\documentclass{article}
\PassOptionsToPackage{hyphens}{url}
\usepackage[accepted]{resource/icml2024}
\icmltitlerunning{The Fundamental Limits of Least-Privilege Learning}
\usepackage[pdftex,bookmarksnumbered,bookmarksopen,colorlinks,citecolor={blue!70!black},linkcolor={blue!70!black},urlcolor=blue]{hyperref}
\usepackage{amsthm, amssymb}
\usepackage{xspace}
\usepackage{mathtools}
\usepackage{microtype}
\usepackage{graphicx}
\usepackage{tikz}
\usetikzlibrary{positioning, calc, shapes.geometric, shapes, shapes.multipart, arrows.meta, arrows, decorations.markings, external, trees}
\tikzstyle{line}=[draw]
	\setkeys{Gin}{width=\textwidth, totalheight=\textheight, keepaspectratio}
  	\graphicspath{{graphics/}}
\usepackage{caption}
\usepackage{subcaption}

\usepackage{booktabs} 
\usepackage{nicefrac}
\usepackage{tabularx}
\usepackage{wrapfig}
\usepackage{fontawesome}
\usepackage[capitalize]{cleveref}

\usetikzlibrary{patterns}
\usepackage{scalefnt}
\usepackage{pgfplots}
\pgfplotsset{compat = newest}
\usetikzlibrary{tikzmark,shapes.misc}

\providecommand{\UsePackageFor}[2]{ \ifx#2\undefined\usepackage{#1}\fi }

\UsePackageFor{amssymb}{\boxtimes}
\usepackage[normalem]{ulem}		
\usepackage{bbding}				
\usepackage{endnotes}			
\usepackage{hyperref}			
\usepackage{resource/comment/hyperendnotes}	
\usepackage{pdfcolfoot}			
\usepackage{xcolor}

\ifdefined\OrigFootnote\relax\else
	\newenvironment{FootnoteContent}{}{}
	\let\OrigFootnote\footnote
	\let\OrigFootnoteText\footnotetext
	\renewcommand{\footnotetext}[1]{\OrigFootnoteText{\begin{FootnoteContent}#1\end{FootnoteContent}}}
	\renewcommand{\footnote    }[1]{\OrigFootnote    {\begin{FootnoteContent}#1\end{FootnoteContent}}}
\fi

\definecolor{PurplePlum}{rgb}{0.1,0,0.55}
\definecolor{Brown}{rgb}{0.5,.25,0}
\definecolor{Orange}{rgb}{1,.3,0}
\definecolor{Gray}{rgb}{.7,.7,.7}
\definecolor{DarkGreen}{rgb}{.1,.41,.1}
\definecolor{Turquoise}{HTML}{00CED1}

\newif\ifBleck
\newcommand\Bleck {\Blecktrue} 

\newcommand\Colour[1] {\color{#1}}

\newcommand\PrintToCLinks{	
  {\Colour{blue}\mbox{
    \hyperlink{w1619}{\sf$\rightarrow$~top}\quad
    \hyperlink{w1031}{\sf$\rightarrow$~toc}\quad
    \hyperlink{w1148}{\sf$\rightarrow$~lof}\quad
    \hyperlink{GreenRoom}{\sf$\rightarrow$~gr}\quad
    \hyperlink{EndNotes}{\sf$\rightarrow$~en}\quad
    \hyperlink{Sargasso}{\sf$\rightarrow$~sg}\quad
    \hyperlink{Index}{\sf$\rightarrow$~idx}
  }}
}
\makeatletter 
\newcommand\ToCLinks{
  \ifx\@onlypreamble\@notprerr		
    \hypertarget{w1619}{}			
  \else
    \AtBeginDocument{\hypertarget{w1619}{}}	
  \fi

  \ifBleck\else	
    \ifdefined\cofoot
      \cofoot{\PrintToCLinks}
      \cefoot{\PrintToCLinks}
    \else
      \def\@oddfoot{\PrintToCLinks}
      \def\@evenfoot{\PrintToCLinks}
    \fi
 \fi
}
\makeatother

\newif\ifEndNotes 

\newcommand\FnSym{{\scriptsize\PencilLeftDown\kern.1em}}		
\newcommand\EnSym {{$\bigtriangledown$}}

\def\MarkupsHowto{} 
\newcommand{\MarkupsHowtoAdd}[1]{\expandafter\def\expandafter\MarkupsHowto\expandafter{\MarkupsHowto{}#1}} 

\newif\ifMarkupsHowtoPrinted 
\newif\ifSuppress 

\newcommand\MakeMarkups[3][.]{

     \Suppressfalse
     \ifBleck\Suppresstrue\fi
     \ifx0#1\Suppresstrue\fi
     \ifx1#1\Suppressfalse\fi

     \expandafter\providecommand\csname#2x\endcsname {} 
     \ifSuppress\expandafter\renewcommand\csname#2x\endcsname{\relax}\else
                       \expandafter\renewcommand\csname#2x\endcsname{#3}\fi

     \expandafter\providecommand\csname#2\endcsname {} 
     \ifSuppress\expandafter\renewcommand\csname#2\endcsname[1]{##1}\else
                       \expandafter\renewcommand\csname#2\endcsname[1]{{\csname#2x\endcsname##1}}\fi

     \expandafter\providecommand\csname#2d\endcsname {} 
     \ifSuppress\expandafter\renewcommand\csname#2d\endcsname[1]{\relax}\else
                       \expandafter\renewcommand\csname#2d\endcsname[1]{{\csname#2x\endcsname\sout{##1}}}\fi

     \expandafter\providecommand\csname#2r\endcsname {} 
     \ifSuppress\expandafter\renewcommand\csname#2r\endcsname[2]{{##2}}\else
                       \expandafter\renewcommand\csname#2r\endcsname[2]{\csname#2d\endcsname{##1} \csname#2\endcsname{##2}}\fi

     \expandafter\providecommand\csname#2i\endcsname {} 
     \ifSuppress\expandafter\renewcommand\csname#2i\endcsname[1]{\relax}\else
                       \expandafter\renewcommand\csname#2i\endcsname[1]{\csname#2\endcsname{##1}}\fi

     \expandafter\providecommand\csname#2t\endcsname {} 
     \ifSuppress\expandafter\renewcommand\csname#2t\endcsname[1]{\relax}\else
                       \expandafter\renewcommand\csname#2t\endcsname[1]{{\csname#2x\endcsname{\mbox{$\langle\!\langle$}##1{\csname#2x\endcsname\mbox{$\rangle\!\rangle$}}}}}\fi

     \expandafter\providecommand\csname#2b\endcsname {} 
     \ifSuppress\expandafter\renewcommand\csname#2b\endcsname[1][empty]{\relax}\else
                       \expandafter\renewcommand\csname#2b\endcsname[1][\empty]{\ifx\empty##1\empty
                       	\label{#2-bookmark} 
                              \marginpar [\raggedleft\csname#2\endcsname{{\footnotesize\fbox{#2 working here}}~$\Longrightarrow$}]
                                                {\csname#2\endcsname{$\Longleftarrow$~{\footnotesize\fbox{#2 working here}}}}
                       \else 
                       	\marginpar [\raggedleft\csname#2\endcsname{\ifx\empty##1\empty\else\fbox{\tiny\parbox{8em}{\raggedright##1}}~\fi$\Longrightarrow$}]
                                                {\csname#2\endcsname{$\Longleftarrow$\ifx\empty##1\empty\else~{\tiny\fbox{\parbox{8em}{\raggedright##1}}}\fi}}\fi}\fi

     \expandafter\providecommand\csname#2TD\endcsname {} 
     \ifSuppress\expandafter\renewcommand\csname#2TD\endcsname{\relax}\else
                       \expandafter\renewcommand\csname#2TD\endcsname{\csname#2\endcsname{\fbox{#2 to do}}}\fi

     \expandafter\providecommand\csname#2Bar\endcsname {} 
     \ifSuppress\expandafter\renewcommand\csname#2Bar\endcsname{\relax}\else
                       \expandafter\renewcommand\csname#2Bar\endcsname{\csname#2\endcsname{\scriptsize\XSolidBrush}}\fi

     \expandafter\providecommand\csname#2f\endcsname {} 
     \ifSuppress\expandafter\renewcommand\csname#2f\endcsname[2][]{\relax}\else
      \expandafter\renewcommand\csname#2f\endcsname[2][\empty]{ 
        {\mbox{\csname#2x\endcsname\tiny$\boxtimes$}\marginpar{\hsize1cm\csname#2x\endcsname\fbox{\FnSym\footnotemark}}\relax
        \footnotetext{\csname#2x\endcsname##2}}}\fi

     \expandafter\providecommand\csname#2e\endcsname {}
     \ifSuppress\expandafter\renewcommand\csname#2e\endcsname[1]{\relax}\else%
      \expandafter\renewcommand\csname#2e\endcsname[1]{%
       \global\EndNotestrue
       \mbox{\scriptsize\csname#2x\endcsname$\boxtimes$}\relax%
       \marginpar{\hsize1cm\csname#2x\endcsname\fbox{\EnSym\endnotemark%
                          \hypertarget{ENmark\thepage-\theendnote}{}~\hyperlink{ENtext\thepage-\theendnote}{{\Colour{blue}$\downarrow$}}}%
       }%
       {
        \def\zz{\noexpand#3}%
        \edef\z{~{[Endnote \theendnote\ %
        on p.\noexpand\hypertarget{ENtext\thepage-\theendnote}{}\thepage%
                    ~\noexpand\hyperlink{ENmark\thepage-\theendnote}{{\noexpand\Colour{blue}$\uparrow$}}]}%
        }%
        \expandafter\endnotetext\expandafter{\z\vspace{2ex}\\ ##1\newpage}%
       }
      }\fi

     \expandafter\providecommand\csname#2n\endcsname {}
     \ifSuppress\expandafter\renewcommand\csname#2n\endcsname[1]{\relax}\else%
      \expandafter\renewcommand\csname#2n\endcsname[1]{%
       \global\EndNotestrue
    \marginpar{{\tiny\endnotemark}\hypertarget{ENmark\thepage-\theendnote}{}~\hyperlink{ENtext\thepage-\theendnote}{}}
       {
        \def\zz{\noexpand#3}%
        \edef\z{~{\zz[Endnote (deferred) 
        from p.\noexpand\hypertarget{ENtext\thepage-\theendnote}{}\thepage%
        ]}%
        }%
        \expandafter\endnotetext\expandafter{\z\vspace{2ex}\\ ##1\newpage}%
       }
      }\fi

     \expandafter\providecommand\csname#2fe\endcsname {} 
     \ifSuppress\expandafter\renewcommand\csname#2fe\endcsname[2][]{\relax}\else 
      \expandafter\renewcommand\csname#2fe\endcsname[2][]{ 
       \def\File{##1}\relax
       \ifx\File\empty\csname#2f\endcsname{##2}\else 
        \global\EndNotestrue 
        \mbox{\scriptsize\csname#2x\endcsname$\boxtimes$}
        \marginpar{\csname#2x\endcsname\fbox{\FnSym\footnotemark}}\relax
        \footnotetext{~\csname#2x\endcsname##2\
                             --- See [\EnSym\endnotemark\hypertarget{ENmark\thepage-\theendnote}{}
                             \kern-.2em\hyperlink{ENtext\thepage-\theendnote}{{\Colour{blue}$\downarrow$}}].}\relax
       { 
         \def\zz{\noexpand#3}
         \edef\z{~{\zz[Endnote~\thefootnote~on~p.\noexpand\hypertarget{ENtext\thepage-\theendnote}{}\thepage
                     ~\noexpand\hyperlink{ENmark\thepage-\theendnote}
                     {{\noexpand\Colour{blue}\kern-0.1em$\uparrow$}]}}
                     {\noexpand\footnotesize\noexpand\newline\noexpand\hspace*{2em} (~from file {\noexpand\tt\File.tex}~)}
         }
         \expandafter\endnotetext\expandafter{\z~\par\input{##1}\newpage}
        } 
       \fi 
      } 
     \fi 

     \ifSuppress\relax\else\ifBleck\relax\else
      \MarkupsHowtoAdd{\par\csname#2t\endcsname{
       $\backslash$\texttt{#2}$\cdots$\ markups are in \textbf{this} colour\ifx#1..\else\ifx1#1.\else, e.g.\ for #1.\fi\fi
       \ifMarkupsHowtoPrinted\relax\else 
        \global\MarkupsHowtoPrintedtrue 
        \begin{quote}\begin{tabular}{l@{\hspace{2em}}p{.7\linewidth}}
         \multicolumn{2}{l}{\texttt{$\backslash$MakeMarkups\ifx#1.\relax\else[#1]\fi\{#2\}\{{\it$\langle$colour command\/$\rangle$}\}}
         				 --- Defines the macros below:}\\
             & see comments at \texttt{$\backslash$MakeMarkups} definition. \\[1ex]
         \texttt{$\backslash$#2\{$\langle$text$\rangle$\}} & Sets \texttt{$\langle$text$\rangle$} in \texttt{#2}'s colour. \\
         \texttt{$\backslash$#2x} & Changes to \texttt{#2}'s colour (until end of context). \\
         \texttt{$\backslash$#2d\{$\langle$text$\rangle$\}} & Sets \texttt{$\langle$text$\rangle$} in \texttt{#2}'s colour with a strikethrough (i.e.\ delete). \\
         \texttt{$\backslash$#2r\{$\langle$this$\rangle$\}\{$\langle$that$\rangle$\}} &
          Strikes through \texttt{$\langle$this$\rangle$} and inserts \texttt{$\langle$that$\rangle$} (i.e.\ replace). \\
         \texttt{$\backslash$#2f\{$\langle$text$\rangle$\}} & Meta-comment: puts \texttt{$\langle$text$\rangle$} in a \texttt{#2}-footnote with a {\tiny$\boxtimes$} in the main text. \\
         \texttt{$\backslash$#2t\{$\langle$text$\rangle$\}} & Use for meta when  \texttt{$\backslash$#2f} isn't allowed (``Not in outer-par mode.'') \\
         \texttt{$\backslash$#2b[$\langle$optional$\rangle$]} & Marginal pointer, with label for hyper-linking directly there. \\
         \texttt{$\backslash$#2e\{$\langle$text$\rangle$\}} & Puts \texttt{$\langle$text$\rangle$} in a \texttt{#2}-endnote with a (big) $\boxtimes$ in the main text. \\[.5ex]
         \texttt{$\backslash$#2n\{$\langle$text$\rangle$\}} & Like \texttt{$\backslash$#2e}
         except there's no reference from the main text. Good for ``decluttering''
         when you still want to have the footnote- or endnote texts as reminders. \\[.5ex]
         \texttt{$\backslash$#2fe[$\langle$this$\rangle$]\{$\langle$that$\rangle$\}} & Makes a \texttt{$\backslash$#2f\{$\langle$that$\rangle$\}} that refers to a \\
           & \texttt{$\backslash$#2e\{$\langle$contents of file this.tex$\rangle$\}}. \\
           & Without the optional argument, acts as \texttt{$\backslash$#2f\{$\langle$that$\rangle$\}}. \\[.5ex]
         \texttt{$\backslash$#2Bar} & Inserts ``burn after reading'' symbol \csname#2Bar\endcsname, meaning
          \begin{quote}\begin{itemize}\setlength\itemsep{0pt}
           \item If yours is the only \csname#2Bar\endcsname\ in this (presumably someone else's) footnote, and you are happy that the footnote has been addressed,
           go ahead and comment-out the whole footnote. (The \csname#2Bar\endcsname\ is their request for you to ``approve and remove''.)
           \item If you are not happy, delete only your \csname#2Bar\endcsname\ and follow-on in the footnote
            (in your colour, i.e.\ with \texttt{$\backslash$#2x}) saying why you are not happy.
           \item If you are happy, but there are others' burn-after-reading symbols as well as yours, just delete yours; the other people have not yet responded.
          \end{itemize}
          \end{quote}
          The idea is that when everyone's happy, the last person will comment-out the meta-text. \\[0.5ex]
         \texttt{$\backslash$#2TD} & Inserts {\csname#2TD\endcsname}\ . \\
        \end{tabular}\end{quote}
       \fi
      }}
     \fi\fi
}

\newif\ifNoGreenRoom
\newcommand\MakeGreenRoom {\ifBleck\relax\else\ifNoGreenRoom\relax\else
\newcommand\NewGRLabel[1] {\OldGRLabel{GreenRoom-##1}} 
 \newcommand\NewGRRef[1] 
 {\expandafter\ifx\csname r@GreenRoom-##1\endcsname\relax\OldGRRef{##1}\else\OldGRRef{GreenRoom-##1}\fi}
 \let\OldGRLabel\label \let\label\NewGRLabel
 \let\OldGRRef\ref \let\ref\NewGRRef
 \hrule
 ~\\\begin{center}\Huge \hypertarget{GreenRoom}{Green Room}
 \end{center}~\\
 \hrule
\fi\fi}

\newcommand\EndGreenRoom  {\ifBleck\relax\else\ifNoGreenRoom\relax\else
\let\label\OldGRLabel
\let\ref\OldGRRef
\fi\fi}

\newif\ifNoEndNotes

\newif\ifNoSargasso
\newcommand\MakeSargasso {
 \hypertarget{Sargasso}{}
 \newcommand\NewLabel[1] {\OldLabel{Sargasso-##1}} 
 \newcommand\NewRef[1] 
 {\expandafter\ifx\csname r@Sargasso-##1\endcsname\relax\OldRef{##1}\else\OldRef{Sargasso-##1}\fi}
 \let\OldLabel\label \let\label\NewLabel
 \let\OldRef\ref \let\ref\NewRef
\ifBleck\end{document}\else\ifNoSargasso
\relax
\else
  \hrule
  ~\\\begin{center}\Huge Sargasso
  \end{center}~\\
  \hrule
 \fi\fi
}

\newcommand\EndSargasso  {\ifBleck\relax\else\ifNoSargasso\relax\else
\let\label\OldLabel
\let\ref\OldRef
\fi\fi}

\newcommand\EndDocument {\ifBleck\end{document}\fi} 

\newcommand\Cite[2][\empty] {{\Colour{red}\ifx#1\empty[#2]\else[#2,~#1]\fi}}

\MakeMarkups[Carmela]{C}{\Colour{Orange}}
\MakeMarkups[Nicolas]{N}{\Colour{PurplePlum}}
\MakeMarkups[Theresa]{T}{\Colour{DarkGreen}}
\MakeMarkups[Bogdan]{B}{\Colour{Turquoise}}
\Bleck    

\newcommand{\para}[1]{\vspace{1mm}\noindent\textbf{#1.}}

\newcommand{\subfig}[1]{\xspace(\textit{#1})\xspace}
\newcommand{\figurewidth}{\textwidth}

\newtheorem{proposition}{Proposition}
\newtheorem{definition}{Definition}
\newtheorem{theorem}{Theorem}
\newtheorem{corollary}{Corollary}
\newtheorem{lemma}{Lemma}
\newtheorem{remark}{Remark}

\newtheorem{assumption}{Assumption}
\crefname{assumption}{Assumption}{Assumptions}

\newcommand\independent{\protect\mathpalette{\protect\independenT}{\perp}}
\def\independenT#1#2{\mathrel{\rlap{$#1#2$}\mkern2mu{#1#2}}}

\newcommand{\define}{~\triangleq~}
\newcommand{\CNN}{\textsc{CNN256}\xspace}
\newcommand{\ResNet}{\textsc{ResNet18}\xspace}

\newcommand{\dataset}{D}
\newcommand{\dataDist}{P}

\newcommand{\joint}{{\inputFeatures,\labelY,\labelS}}
\newcommand{\numSamples}{N}

\newcommand{\labelS}{S}
\newcommand{\labelY}{Y}
\newcommand{\labelSVal}{s}
\newcommand{\labelYVal}{y}
\newcommand{\labelYSpace}{\mathbb{Y}}
\newcommand{\labelSSpace}{\mathbb{S}}
\newcommand{\record}{\mathbf{r}}
\newcommand{\inputFeatures}{X}
\newcommand{\inputFeaturesVal}{x}
\newcommand{\inputSpace}{\mathbb{X}}
\newcommand{\model}{f}
\newcommand{\modelFeatures}{Z}
\newcommand{\modelFeaturesVal}{z}
\newcommand{\modelFeaturesSpace}{\mathbb{Z}}
\newcommand{\generic}{W}
\newcommand{\genericSpace}{\mathbb{W}}

\newcommand{\utilInfo}{I_\alpha}
\newcommand{\gain}{I_\infty}
\newcommand{\gainEst}{\smash{\tilde{I}_\infty}}
\newcommand{\maxLeakage}{\mathcal{L}}

\newcommand{\supp}{\mathrm{supp}}

\definecolor{plotBlue}{RGB}{136,204,238}
\definecolor{plotRed}{RGB}{204,102,119}
\definecolor{plotYellow}{RGB}{221, 204, 119}
\definecolor{plotPurple}{RGB}{51,34,136}
\definecolor{plotGreen}{RGB}{17,119,51}
\definecolor{plotPink}{RGB}{170,68,153}

\Bleck

\begin{document}
\twocolumn[

    \icmltitle{The Fundamental Limits of Least-Privilege~Learning}

    \icmlsetsymbol{equal}{*}
    
    \begin{icmlauthorlist}
        \icmlauthor{Theresa Stadler}{epfl,equal}
        \icmlauthor{Bogdan Kulynych}{chuv,equal}
        \icmlauthor{Michael C. Gastpar}{epfl}
        \icmlauthor{Nicolas Papernot}{vector}
        \icmlauthor{Carmela Troncoso}{epfl}
    \end{icmlauthorlist}
    
    \icmlaffiliation{epfl}{EPFL, Lausanne, Switzerland}
    \icmlaffiliation{chuv}{Lausanne University Hospital \& University of Lausanne, Switzerland}
    \icmlaffiliation{vector}{University of Toronto \& Vector Institute, Toronto, Canada}
    
    \icmlcorrespondingauthor{Theresa Stadler}{theresa.stadler@epfl.ch}

    \icmlkeywords{Privacy, Trustworthy Machine Learning}

\vskip 0.3in

]
\printAffiliationsAndNotice{\icmlEqualContribution}

\begin{abstract}
The promise of least-privilege learning -- to find feature representations that are useful for a learning task but prevent inference of any sensitive information unrelated to this task -- is highly appealing. However, so far this concept has only been stated informally. It thus remains an open question whether and how we can achieve this goal.
In this work, we provide the \emph{first formalisation of the least-privilege principle for machine learning} and characterise its feasibility.
We prove that there is a \emph{fundamental trade-off} between a representation's utility for a given task and its leakage beyond the intended task: it is not possible to learn representations that have high utility for the intended task but, at the same time prevent inference of any attribute other than the task label itself. This trade-off holds under realistic assumptions on the data distribution and {regardless} of the technique used to learn the feature mappings that produce these representations.
We empirically validate this result for a wide range of learning techniques, model architectures, and datasets.
\end{abstract}

\section{Introduction}
\label{sec:intro}
The need to reveal data to untrusted service providers to obtain value from machine learning as a service
(MLaaS) puts individuals at risk of data misuse and harmful inferences. The service provider observes raw data records at training or inference time and might abuse them for purposes other than the intended learning task. For instance, an image shared with a provider for the purpose of face verification might be misused to infer an individual's race and lead to discrimination~\citep{CitronS22}.

\para{Sharing Representations to Prevent Data Misuse} Collaborative learning and model partitioning claim to prevent such misuse in model training and inference MLaaS settings, respectively.
In both cases, individuals share a feature representation of their raw data with the service provider; in the form of model updates in the collaborative learning setting~\citep{FedGoogle, FedApple} and of feature encodings in the model partitioning setting~\citep{OsiaTS18, ChiOY18, WangZB18, BrownMM22}.
Proponents of both techniques argue that, because individuals \emph{only share a representation of their data}, and not the data itself, the service provider no longer has access to information that might be abused for purposes other than the intended task.

\para{Unintended Feature Leakage} However, previous research shows that a passive adversary can misuse the shared representations to infer data attributes that are unrelated to the learning task, or even reconstruct data records~\citep{BoenischDS23,GanjuWY18,MelisSDS19,SongS19}.
For instance, \citet{SongS19} show that features extracted from a gender classification model also reveal an individual's race.
Even higher-layer features, that are assumed to be more learning-task specific, might lead to such unexpected inferences~\citep{MoBMH21}.
These examples show that limiting data access to feature representations does not necessarily prevent unintended information leakage and thus, does not fully mitigate the risk of data misuse associated with attributes other than the learning task.

\para{Least-Privilege Learning} Some works~\citep{OsiaTS18, MelisSDS19, BrownMM22} suggest that the solution to this issue is to ensure that the feature mappings that produce such representations follow the \emph{least-privilege principle} (LPP). That is, to enforce that the representations shared with the service provider  \emph{only contain information relevant to the learning task, and nothing else}.
Previous work repeatedly suggests least-privilege learning (LPL) as a promising avenue to prevent data misuse. Yet, it has only been described informally and lacks a precise definition.

\para{Contributions} We make the following contributions:
\begin{enumerate}
    \item We provide the first formalisation of the least-privilege principle for machine learning as a variant of the generalized Conditional Entropy Bottleneck problem~\citep{Fisher20}. Our formalisation enables us to characterise the limits of unintended information leakage in MLaaS settings.
    
    \item We observe that any feature representation must at least leak all information about the input data that can be inferred from the learning task label itself. 
    We show experimentally that this \emph{fundamental leakage} can reveal information that is not intuitively related to the intended task and that could be misused for harmful inferences.
    
    \item We formally prove a fundamental trade-off: under realistic assumptions on the data distribution, it is not possible to learn feature representations that have high utility for the intended task but, at the same time, restrict information leakage about data attributes other than the intended task to its fundamental leakage.
    
    \item We experimentally demonstrate this trade-off across learning techniques, model architectures, and datasets. We show that as long as the representations have utility for their intended task, there exist attributes other than the task label that can be inferred from the shared representations and thus violate the least-privilege principle.
\end{enumerate}

\para{Related Work} Prior works refer to the LPP but so far only study the problem of \emph{attribute obfuscation} for a \emph{single, fixed sensitive attribute}.
For instance, \citet{SongS19} experimentally observe that current censoring techniques can only prevent a model from learning a sensitive attribute at a cost in model utility. Their claim that ``overlearning [of sensitive attributes] is intrinsic''~\citep{SongS19} to machine learning models is  derived from empirical observations on a single learning task and sensitive attribute for a small set of censoring techniques.
\citet{ZhaoCTG20} formally derive a lower bound on the trade-off between a model's performance on its intended learning task and hiding a \emph{fixed sensitive attribute}. 

While these results characterise the trade-off between a representation's utility and leakage with respect to a specific attribute, they do not actually evaluate whether it is possible to learn useful representations that fulfil the LPP. They focus on a single attribute, whereas the LPP demands that the representations shared with the service provider should prevent inference of \emph{any information} other than the intended task~\citep{MelisSDS19, BrownMM22}. Furthermore, neither \citet{SongS19} nor \citet{ZhaoCTG20} consider that useful representations must at least leak any sensitive information already revealed by the task label itself. As we argue in \cref{sec:lpl}, such \emph{fundamental leakage} is crucial to consider when formalising LPL.
\section{Problem Setup}
\label{sec:background}

We consider the common MLaaS setting in which individual data owners, or users, share their data with a service provider for a target task, e.g., model training or inference. Users agree to the usage of their data for the intended task but want to prevent data misuse through the service provider. In this section, we formalize this problem setup.

\para{Notation} Let $\inputFeatures, \labelY, \labelS \sim \dataDist_\joint$ be a set of random variables distributed according to $\dataDist_\joint$ where $\inputFeatures \in \inputSpace$, and $\labelY \in \labelYSpace$ are, respectively, an \emph{example} and its \emph{learning task label}, and $\labelS \in \labelSSpace$ is a \emph{sensitive attribute}.
For any three random variables $\inputFeatures, \labelY, \generic$, we denote by $\labelY -
\inputFeatures - \generic$ \emph{a Markov chain}, which is equivalent to stating that:
$\labelY \independent \generic \mid \inputFeatures$.

\para{Assumptions on the Data Distribution} To make our formal analyses tractable, we assume that the spaces $\inputSpace, \labelYSpace, \labelSSpace$ are discrete and finite; and that the data domain is non-trivial ($|\inputSpace| > 1$ and $|\labelYSpace| > 1$) and has full support.
%
We further make the following assumption about the data distribution:
\begin{assumption}[Strictly positive posterior]
    \label{ass:positivePosterior}
    We say that the \emph{posterior distribution}, $\dataDist_{\labelY \mid \inputFeatures}$, is strictly positive if
    for any $\inputFeaturesVal \in \inputSpace, \labelYVal \in \labelYSpace$ we have $\dataDist_{\labelY \mid \inputFeatures}(\labelYVal \mid
    \inputFeaturesVal) >
    0$.
\end{assumption}
\vspace*{-2mm}
This assumption is realistic in settings where there exists inherent uncertainty about the ground
truth label of a given example. Examples include the presence of \emph{label
noise} introduced by the labelling process~\citep{SongKPSL22}, and, under the Bayesian
interpretation of probability, task labels that are subjective. This is the case, for instance, in many MLaaS applications, such as emotion
recognition or prediction of face attributes (e.g., smiling) which come with uncertainty and labelling unreliability~\citep{RajiBPDH21}.

\para{Problem Setup} In the MLaaS setting, individual \emph{users} hold a set of \emph{data records} $(\inputFeaturesVal_i, \labelYVal_i)$ for $i = 1,\ldots,\numSamples$ and agree to share their data with a \emph{service provider} for a specific \emph{purpose}.
For example, a user might be willing to reveal an image of their face to the provider of a biometric face recognition system to prove their identity.
However, if the user directly reveals the original image $\inputFeatures$, the service provider can \emph{misuse the shared data} for purposes other than the intended task, i.e., to verify the user's identity.

To prevent such data misuse, users share a \emph{representation} of their data that restricts information leakage to the necessary minimum for the \emph{intended purpose} (see \cref{fig:problem_setup}). For instance, in the previous example, a representation that only reveals features relevant to recognising a user's identity. The goal hence is to find a possibly randomized \emph{feature mapping} $\model_E(\cdot)$ that maps input $X$ to a representation $\modelFeatures = \model_E(X)$ that both has \emph{high utility for the intended task} $\labelY$ but \emph{prevents inference of any other data attribute} $\labelS$.
We formalise this goal and characterise its feasibility in the next section.

\begin{figure*}
    \centering
    \resizebox{0.65\linewidth}{!}{
    \input{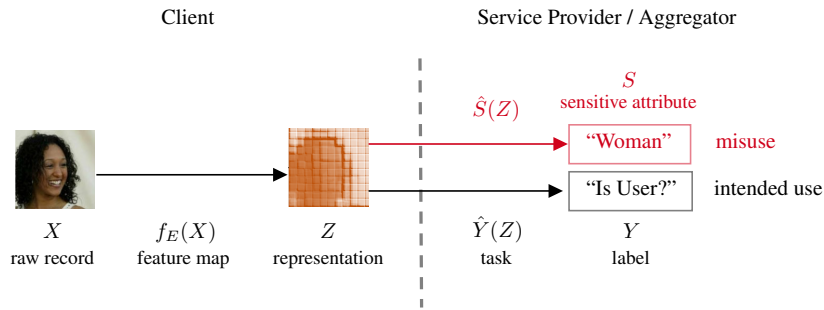}
    }
    \caption{To prevent potential data misuse in a MLaaS setting, users share a representation of their data. These representations should be useful to achieve the intended purpose (verify users' identity) but prevent inference of other data attributes (users' gender) that might lead to harms, such as discrimination.}
    \label{fig:problem_setup}
\end{figure*}

\para{Training vs. Inference Settings} The example above describes a MLaaS \emph{inference} setting in which users directly share a representation of their data for the purpose of predicting the intended task label. However, the same problem setup applies to the collaborative MLaaS \emph{training} setting in which users share gradient updates instead of their raw data to prevent data misuse. As \citet{MelisSDS19} show, these ``gradient updates can [...] be used to infer feature values'' and thus are just a version of the model-generated representations $\modelFeatures = \model_E(\inputFeatures)$.

\section{The Least-Privilege Principle in Machine Learning}
\label{sec:lpl}

\para{The Least-Privilege Security Principle} The LPP is a design principle for building secure
information systems introduced by \citet{SaltzerS75}. Its definition demands that ``Every
program and every user of the system should operate using the least set of privileges necessary to
complete the job.'' In secure systems engineering, a privilege is a clearly defined action that an actor in the system is authorised to carry out. Transferring this concept to the machine learning domain is not trivial. To do so, we have to first quantise the learning process into a set of smaller privileges and then define what is the minimum set of privileges needed to carry out a learning task.

We approach this problem through the lens of \emph{attribute inference}: The goal of the user is to allow inference of the task attribute $\labelY$ from representation $\modelFeatures$ but prevent inference of any data attribute $\labelS$ other than the learning task. We define a \emph{privilege as the ability to learn the value of a particular attribute}, and formalise the LPP in terms of inference gain about data attributes other than the learning task.

\subsection{Utility Measure}
The representation shared with the service provider should be useful for the intended purpose, i.e., contain features relevant to correctly infer task label $\labelY$. One common way to formalise this objective is the \emph{mutual information} between the representations and the task label $I(\labelY; \modelFeatures) = I(\labelY; \model_E(\inputFeatures))$~\citep{AlemiFDM16}.
The higher the information between $\modelFeatures$ and $\labelY$, the more useful the representation will be for the intended task.
We make this notion more general by using a variant of mutual information known as Arimoto's $\alpha$-information~\citep{Arimoto77}, which we denote as $\utilInfo(\labelY; \modelFeatures)$.
In particular, we consider two relevant instantiations of $\alpha$: $\alpha = 1$ and $\alpha = \infty$. In the case that $\alpha = 1$, Arimoto's information is equal to classical Shannon's mutual information $\utilInfo(\labelY; \modelFeatures) = I(\labelY; \modelFeatures)$, described previously. In the case that $\alpha = \infty$, Arimoto's information is the \emph{multiplicative gain} in accuracy of predicting the task label $\labelY$ from the learned representations over baseline guessing~\citep{LiaoKSC19}:
\begin{align}\label{eq:util_gain}
     \gain(\labelY; \modelFeatures) \define \log \frac{\Pr[\labelY = \model_C(\modelFeatures)]}{\Pr[\labelY = \hat \labelY]}
\end{align}
where $\log(\cdot)$ is the base-2 logarithm, $\smash{\model_C(\modelFeatures)}$ denotes the prediction for task label $\labelY$ of a classifier $\model_C$ based on representation $\modelFeatures$, and $\smash{\hat \labelY \define \arg \max_{\labelYVal \in \labelYSpace} \Pr[\labelY = \labelYVal]}$ represents the baseline guess based on the most common attribute value. The latter notion of $I_\infty(Y; Z)$ is especially useful to evaluate utility of representations for classification tasks, as it is normalised prediction accuracy. 

\subsection{Leakage Measure}
\label{subsec:leakage_measure}
As described in \cref{sec:background}, users want to restrict information leakage about data attributes other than the intended task. We evaluate such unintended leakage through the success of an \emph{inference adversary} $\smash{g(W)}$ that tries to infer the value of a sensitive attribute $\labelS$ from information $\generic$. In line with standard practices in security and privacy, we use Bayes-optimal adversaries~\citep{SablayrollesDSO19,ChatzikokolakisCPT23} that achieve optimal inference accuracy and hence measure worst-case inference risk:
\begin{align}
	\label{eq:bayes_adv}
    \hat \labelS(\generic) \define \arg \max_{g:~\genericSpace \rightarrow \labelSSpace} \Pr[\labelS
    = g(\generic)].
\end{align}
We use $\smash{\hat \labelS}$ without the argument to denote the majority class baseline guess: $\smash{\hat \labelS \define \arg \max_{\labelSVal \in \labelSSpace} \Pr[\labelS = \labelSVal]}$.

We measure unintended leakage about sensitive attribute $\labelS$ through an adversary's multiplicative gain (Arimoto's information with $\alpha = \infty$). Formally, the gain of an adversary $\smash{\hat \labelS(\generic)}$ who has access to information $\generic$ over their baseline guess is defined as:
\vspace{-.5em}
\begin{align}
    \gain(\labelS; \generic) \define \log \frac{\Pr[\labelS = \hat
    \labelS(\generic)]}{\Pr[\labelS = \hat \labelS]}.
 \end{align}
We denote the gain of an adversary $\smash{\hat \labelS}(\generic, \generic')$ with access to two sources of information $\generic$ and $\generic'$ over their guess with only one source of information $\generic'$ as:
\vspace{-.2em}
\begin{align}\label{eq:adv_gain}
    \gain(\labelS; \generic \mid \generic') \define \log \frac{\Pr[\labelS = \hat
    \labelS(\generic, \generic')]}{\Pr[\labelS = \hat \labelS(\generic')]}.
 \end{align} 

We use these two measures (\cref{eq:util_gain} as a representation's utility and \cref{eq:adv_gain} as its unintended leakage) to formalise the LPP in machine learning.

\subsection{Strawman Approach: Unconditional Least-Privilege Principle}
The promise of LPL, as suggested by prior work~\citep{OsiaTS18, MelisSDS19, BrownMM22}, is to find a feature representation that only contains information relevant to the ``purpose for which it was designed and \emph{nothing else}''~\citep{BrownMM22}. To formalise the LPP, in contrast to related works on attribute obfuscation, we hence cannot assume the sensitive attribute $\labelS$ to be fixed. Instead, we need to assume that \emph{any attribute} --- e.g., any function of the input --- other than the learning task $\labelS \neq \labelY$ is sensitive and its inference might lead to harm.  
Formally:
\begin{definition}[Unconditional LPP]
\label{def:ulpr}
    Given a data distribution $\dataDist_{\inputFeatures, \labelY}$, a feature map $
    \model_E(\inputFeatures) = \modelFeatures$ satisfies the unconditional LPP with parameter $\gamma$ if
    for any attribute $\labelS \neq \labelY$ which follows the Markov chain $\labelS - \inputFeatures -
    \modelFeatures$, the attribute inference gain is bounded:
    \begin{align}
        \frac{\Pr[\labelS = \hat \labelS(\modelFeatures)]}{\Pr[\labelS = \hat \labelS]} \leq 2^\gamma
    \end{align}
    Equivalently:
    \vspace{-.5em}
    \begin{align}
        \sup_{\labelS \neq \labelY:~\labelS - \inputFeatures - \modelFeatures} \gain(\labelS; \modelFeatures) \leq \gamma
    \end{align}
\end{definition}

Many previous works suggest that it is possible to find a feature map $\model_E(\inputFeatures) = \modelFeatures$ that fulfils the unconditional LPP, and at the same time produces representations with high utility for the learning task~\citep{OsiaTS18, MelisSDS19, BrownMM22}. To address the question whether it is possible to achieve high utility for the intended task and simultaneously satisfy the unconditional LPP, we  formally characterise this trade-off: 

\begin{theorem}[Unconditional LPP and Utility Trade-Off]
\label{th:ulpl_impossibility}
\normalfont
Suppose that $\dataDist_{\labelY \mid \inputFeatures}$ is strictly positive
(\cref{ass:positivePosterior}).
Then, for $\alpha \in \{1, \infty\}$, the following two properties cannot hold at the same time:
\begin{enumerate}
    \item $\modelFeatures = \model_E(\inputFeatures)$ satisfies the unconditional LPP with parameter $\gamma$
    \item $\utilInfo(\labelY; \modelFeatures) > \gamma$
\end{enumerate}
\end{theorem}

We provide a proof of this and all of the following formal statements in \cref{sec:proofs}.
This result implies that whenever a representation has a certain utility for the learning task with $\utilInfo(\labelY, \modelFeatures) > \gamma$, there exists a sensitive attribute for which an adversary's inference gain is at least as large with $\gain(\labelS, \modelFeatures) \geq \gamma$.
In fact, under \cref{def:ulpr}, it is easy to construct this attribute to be infinitesimally close to the task label $\labelY$ but not quite match it. In the next section, we provide an alternative formalisation
of the LPP that captures the requirement $\labelS \neq \labelY$ yet precludes these cases.

\subsection{Formalisation of the Least-Privilege Principle}
In the previous section, we show that the unconditional LPP, i.e., hiding \emph{all} information about $\inputFeatures$, comes with a stringent trade-off. But this is an unnecessarily restrictive goal.
To use a service, users must be willing to reveal to the service provider at least the intended result of the computation, i.e., the task label $\labelY$.
As a consequence, they cannot conceal from the service provider \emph{any information that can be inferred from $\labelY$ itself}.
This information hence defines the \emph{least privilege} that can be given to the service provider, i.e., the minimum access to data attributes that must be granted to carry out a task.
We call this information the \emph{fundamental leakage} of the task. For a given attribute $\labelS$, the fundamental leakage equals $\Pr[\labelS = \smash{\hat \labelS(\labelY)}]$.
We propose a formalisation of the LPP that only demands that sharing a record's feature representation
$\modelFeatures = \model_E(\inputFeatures)$ does not reveal \emph{more information} about a sensitive
attribute $\labelS$ than publishing $\labelY$ itself:
\begin{definition}[LPP]
\label{def:lpr}
    Given a data distribution $\dataDist_{\inputFeatures, \labelY}$, a feature map $
    \model_E(\inputFeatures) = \modelFeatures$ satisfies the LPP with parameter $\gamma$ if
    for any attribute $\labelS$ which follows the Markov chain $\labelS - (\inputFeatures, \labelY) - \modelFeatures$,
    the attribute inference gain from observing $(\modelFeatures, \labelY)$ over the
    fundamental leakage is bounded:
    \begin{align}
        \frac{\Pr[\labelS = \hat \labelS(\modelFeatures, \labelY)]}{\Pr[\labelS = \hat
        \labelS(\labelY)]} \leq 2^\gamma
    \end{align}
    Equivalently:
    \vspace{-.5em}
    \begin{align}
        \sup_{\labelS:~\labelS - (\inputFeatures, \labelY) - \modelFeatures} \gain(\labelS;
        \modelFeatures \mid \labelY) \leq \gamma
    \end{align}
\end{definition}
Notably, the quantity constrained by the LPP is known as \emph{maximal leakage}~\citep{IssaWK19}:
\begin{align}
    \label{eq:max_leakage}
    \maxLeakage(\inputFeatures \rightarrow \modelFeatures \mid \labelY) \define
    \sup_{\labelS:~\labelS - (\inputFeatures, \labelY) - \modelFeatures} \gain(\labelS; \modelFeatures \mid \labelY).
\end{align}
In comparison to the unconditional variant, this formalisation does not require $\labelS \neq \labelY$. Therefore, it does not restrict the adversary's absolute gain from observing $\modelFeatures$, but only restricts the leakage about sensitive attribute $\labelS$ \emph{to its fundamental limit}, i.e., the leakage caused by the learning task label itself.

\para{Interpretation} A feature map that satisfies the LPP in \cref{def:lpr} with a value of $\gamma \approx 0$ restricts the information available to the service provider to what is necessary for the intended purpose of the system: produce an accurate prediction of the task label $\labelY$. Hence, this formalisation supports the data protection principle of \emph{purpose limitation} which states that ``data should only be collected for specified, explicit, and legitimate purposes and not further processed in a manner that is incompatible with those purposes''~\citep{gdpr}.

We must stress that, despite minimizing the information available to the service provider, this definition does not imply that a feature map that satisfies the LPP will necessarily prevent all harmful inferences.
The LPP only limits the risk of revealing $\modelFeatures = \model_E(\inputFeatures)$ to the fundamental risk already incurred by revealing $\labelY$ itself. In \cref{subsec:emp_fundamental}, we empirically show that even the fundamental leakage 
can lead to inferences that might violate users' expectations about the information they reveal through the use of a service---a violation of contextual integrity~\citep{Nissenbaum04}---and lead to harms~\citep{CitronS22}.

\para{Perfect LPP}
We first analyse the feasibility of \emph{perfect LPP} with $\gamma = 0$. Perfect LPP implies that the representation $\modelFeatures$ leaks strictly no more information about the data $\inputFeatures$ than already revealed by the task label $\labelY$ itself. As we show next, perfect LPP is possible only under a restrictive assumption on the data distribution.
\begin{proposition}\label{stmt:perfect-lpp}
    There exists a feature map $\model_E(\inputFeatures)$ that fulfils the LPP with $\gamma = 0$ if and only if we have a Markov chain $\inputFeatures - \labelY - \modelFeatures$.
\end{proposition}


One class of data distributions that fulfils this condition and where perfect LPP is feasible is any distribution in which the intended task label $\labelY$ is a deterministic function of the input data $\inputFeatures$ (see, e.g., \citet{Fisher20}):
\begin{corollary}\label{stmt:det-labels}
    If $\labelY = g(\inputFeatures)$ for some function $g: \inputSpace \rightarrow \labelYSpace$, there exists $f_E(X)$ which achieves perfect LPP with $\gamma = 0$.
\end{corollary}

As we discuss in \cref{sec:background}, in many realistic scenarios, however, there is inherent uncertainty about an example's true label due to, for instance, label noise. Under \cref{ass:positivePosterior}, perfect LPP is unattainable by non-trivial feature maps:
\begin{corollary}\label{stmt:positive-posterior-perfect-lpp}
    Under \cref{ass:positivePosterior}, the feature map $\model_E(\inputFeatures)$ satisfies LPP with $\gamma = 0$ if and only if it is fully random or constant: $\model_E(X) \independent X$.
\end{corollary}

\para{LPP and Utility Trade-Off} We now study whether we can find a feature map that achieves the LPP in \cref{def:lpr} with $\gamma > 0$, and simultaneously has good utility for the intended learning task.
We show the general trade-off between a representation's utility and the LPP requirement:
\begin{theorem}[LPP and Utility Trade-Off]
\label{th:lpl_impossibility}
\normalfont
Suppose that $\dataDist_{\labelY \mid \inputFeatures}$ is strictly positive
(\cref{ass:positivePosterior}).
Then, for $\alpha \in \{1, \infty\}$, the following two properties cannot hold at the same time:
\begin{enumerate}
    \item $\modelFeatures = \model_E(\inputFeatures)$ satisfies the LPP with parameter $\gamma$
    \item $\utilInfo(\labelY; \modelFeatures) > \gamma$
\end{enumerate}
See \cref{sec:proofs} for a full proof and \cref{fig:tradeoff} for an illustration.
\end{theorem}

One way to see why the trade-off holds is through the lens of attributes that reveal maximum possible information about $\inputFeatures$ from $\modelFeatures$ (see \cref{sec:proofs} for a formal description).
Intuitively, we would expect that by allowing for fundamental leakage, i.e., by conditioning on $\labelY$, we might reduce maximal leakage in those cases where $\labelY$ is a maximally revealing attribute. However, as we show in \cref{sec:proofs}, strict positivity of the posterior distribution (\cref{ass:positivePosterior}) prevents the task label $\labelY$ from being a maximally revealing attribute, thus the least-privilege and utility trade-off remains the same as under unconditional LPP (\cref{th:ulpl_impossibility}).

\paragraph{Related Formalisms.}
Our formalisation of the LPP is closely related to the Minimum Necessary Information (MNI) principle introduced by \citet{Fisher20}. The MNI criterion demands that a representation $\modelFeatures$ should maximise the mutual information $I(\labelY ; \modelFeatures)$ while minimising the conditional information $I(\inputFeatures; \modelFeatures \mid \labelY)$. The LPP and MNI differ in their leakage measure, but are equal in the optimal point of $\gamma = 0$ in which $I(\inputFeatures; \modelFeatures \mid \labelY) = \maxLeakage(\inputFeatures \rightarrow \modelFeatures \mid \labelY) = 0$.

The question of \emph{finding} a feature representation that maximises utility but satisfies the LPP posed by \cref{th:lpl_impossibility} can be seen as a variant of the conditional entropy bottleneck (CEB) problem~\citep{Fisher20}, which, in turn is a variant of the standard information bottleneck problem~\citep{AsoodehC20, TishbyPB00}. The CEB problem aims to recover a representation that satisfies the MNI principle.

Unlike in the MNI and CEB, we require a leakage measure that captures the original claim from prior work that it is possible to learn representations that protect against \emph{any} harmful inferences, i.e., protect even against \emph{worst-case inference} adversaries. Maximal leakage satisfies this requirement and in addition has a direct operational meaning in terms of inference gain (see \cref{eq:adv_gain}).

Other works, such as \citet{MireshghallahTJE21} or \citet{MaengGKS24}, differ not only in their leakage measure but crucially do not consider a task's fundamental leakage, i.e., do not define the trade-off in terms of the \emph{conditional} leakage $I_\alpha(\labelS, \modelFeatures \mid \labelY)$. In addition, \citet{MireshghallahTJE21} and other formalisations of related information-theoretic problems, such as the Privacy Funnel~\citep{MakhdoumiSFM14,AsoodehC20}, only consider a fixed sensitive attribute which does not adequately capture the least-privilege requirement.


\begin{figure}
    \centering
    \resizebox{0.5\textwidth}{!}{
\scalefont{1.5}
\begin{tikzpicture}
\begin{axis}[
    xmin = 0, xmax = 2,
    ymin = 0, ymax = 2,
    xtick distance = 1,
    ytick distance = 1,
    grid = both,
    legend pos = south east,
    legend style = {draw=none},
    legend cell align={left},
    minor tick num = 1,
    major grid style = {lightgray},
    minor grid style = {lightgray!25},
    width = \textwidth,
    height = 0.4\textwidth,
    xlabel = {Attribute inference gain, $\gamma = 
         \sup_{\labelS} \log \frac{\Pr[\labelS = \hat
     \labelS(\modelFeatures,\labelY)]}{\Pr[\labelS = \hat \labelS(\labelY)]}$
    },
    ylabel = {Task utility, $\utilInfo(\labelY; \modelFeatures)$},
    ]

\addplot [pattern=north west lines, pattern color=gray!50, domain=0:2, samples=100] {x} \closedcycle;

\addplot[
    domain = 0:2,
    samples = 100,
    smooth,
    ultra thick,
    black,
    ] {x};
\node[fill=white, fill opacity=0.85] at (1,0.5) {Feasible region};

\end{axis}
\end{tikzpicture}}
\caption{$\gamma$-LPP limits maximum utility $\utilInfo(\labelY;\modelFeatures)$ of a representation $\modelFeatures$ to the greyed-out region.
}
\label{fig:tradeoff}
\end{figure}
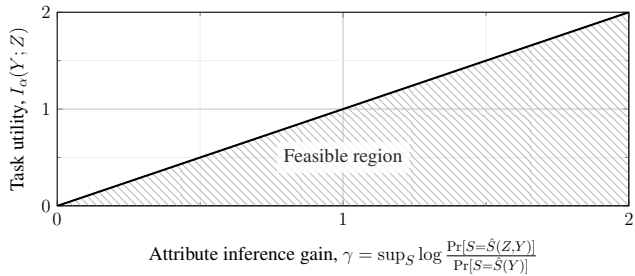

\paragraph{LPP vs. LDP.} \citet{MelisSDS19} suggest that LPL might be used to limit unintended information leakage as an alternative approach to record-level differential privacy~\citep{DworkR14} which they find imposes a high utility cost that prevents its application in many realistic MLaaS scenarios. The hope is that LPL might provide a better trade-off between utility and restricting unintended information leakage than differental privacy. Our formalisation of the LPP enables us to formally analyse this claim.

First, let us define local differential privacy (LDP), a variant of differential privacy that is relevant to our MLaaS setting (see \cref{fig:problem_setup}).
\begin{definition}[LDP]
    A feature map $\model_E(\inputFeatures) = \modelFeatures$ satisfies $\varepsilon$-LDP if for any $x, x' \in \inputSpace$, and any $z \in \modelFeaturesSpace$, we have:
    \begin{equation}
        \frac{\Pr[\model_E(x) = z]}{\Pr[\model_E(x') = z]} \leq 2^\varepsilon.
    \end{equation}
\end{definition}
Conventionally, a power of $e$ instead of $2$ is used, but we use the latter for consistency with our previous definitions.

Given this definition, we show that LDP implies LPP:
\begin{proposition}
    \label{th:ldp-lpp}
    A feature map that satisfies $\varepsilon$-LDP also satisfies $\varepsilon$-LPP.
\end{proposition}
%
%

An important implication is that \emph{a feature map that satisfies the LPP comes with the same utility trade-off as one that satisfies LDP}. Thus, in the setting of \cref{th:lpl_impossibility}, LPL cannot provide an alternative to differential privacy at a lower utility cost as envisioned by prior work~\citep{OsiaTS18, MelisSDS19}.

\para{Takeaways} In summary, \cref{th:lpl_impossibility} implies that in many realistic applications (see \cref{ass:positivePosterior}) there is a stringent trade-off between a representation's utility and achieving the LPP. Notably, this holds for \emph{any} feature representation \textit{regardless of the learning technique used to obtain the feature mapping or the exact model architecture}.
In the next section, we show the practical implications of this result. 
\section{Empirical Evaluation}
\label{sec:empirical}

In this section, we empirically validate our theoretical results.
We demonstrate that the fundamental trade-off between a representation's utility for its intended task and the LPP applies to \emph{any feature representation regardless of the feature learning technique, model architecture, or dataset}. 
Due to space constraints, in the main body of the paper, we only present results for an image dataset, two different learning techniques, and a single model architecture. We discuss additional results that confirm that our theoretical results hold on a wider range of learning techniques, models, and datasets in \cref{subsec:emp_tradeoff}, with their details deferred to \cref{sec:add_exp}.

\para{Data} In our main experiment, we use the LFWA+ image dataset which has multiple binary attribute labels for each image~\citep{LFWA}. The full dataset contains $13,143$ examples which we split in the following way: $20\%$ of records are given to the adversary as an auxiliary dataset $\dataset_A$. The remaining $10,514$ records are split $80/20\%$ across a train $\dataset_T$ and evaluation set $\dataset_E$.

\para{Model} 
We choose a deep convolutional neural network (\CNN) used by \citet{MelisSDS19}, the first work to propose feature learning under a least-privilege principle. In the main results, we use the network's last layer representation as feature map $\modelFeatures =  \model_E(\inputFeatures)$. We choose the last layer because, as \citet{MelisSDS19} and \citet{MoBMH21} show, higher layer representations are expected to be more learning task specific and hence should restrict leakage of data attributes other than the task more than lower layer representations. In \cref{sec:add_exp}, we include experiments with representations from other hidden layers of the model and for a MLaaS training setting in which users share gradients from all layers of the model.

\para{Adversaries} To evaluate leakage of a given attribute, we instantiate Bayes-optimal adversaries with access to the auxiliary set $\dataset_A$ of labelled examples $\record_i = (\inputFeaturesVal_i, \labelYVal_i, \labelSVal_i)$ for $i=1,\ldots, A$.
The label-only adversary $\smash{\hat{\labelS}(\labelY)}$ computes the relative frequency counts over $\dataset_A$ to estimate:
$$\tilde{\Pr}[\labelS=\labelSVal \mid \labelY=\labelYVal] \define 
\frac{\sum_{i=1}^A \mathbf{1}[\labelSVal_i = \labelSVal, \labelYVal_i = \labelYVal]}
{\sum_{i=1}^A \mathbf{1}[\labelYVal_i = \labelYVal]},$$
and outputs a guess according to:
\[\smash{\hat{\labelS}(\labelY=\labelYVal) = 
\arg \max_\labelSVal \tilde{\Pr}[\labelS=\labelSVal \mid \labelY=\labelYVal]}.\]
The features adversary $\hat{\labelS}(\modelFeatures, \labelY)$ is given black-box access to
the feature mapping of the trained model. To collect a train set, the features adversary submits each example $\inputFeaturesVal_i \in \dataset_A$ to the model and
receives back its representation $\modelFeaturesVal_i =
\model_E(\inputFeaturesVal_i)$.
The adversary then trains a Random Forest classifier with $50$ decision trees on the collected samples $(\modelFeaturesVal_i, \labelYVal_i, \labelSVal_i)$ to estimate $\smash{\tilde{\Pr}[\labelS = \labelSVal \mid \modelFeatures=\modelFeaturesVal, \labelY = \labelYVal]}$. We opt for Random Forest as an attack model based on its superior performance over other classifiers we tested.

\para{Experimental Setup} In each experiment, we select one out of $12$ attributes from the LFWA+ dataset as the model's learning task $\labelY$ and a second attribute as the sensitive attribute $\labelS$ targeted by the adversary.
We select attributes for which we expect the distribution $\dataDist_{\labelY \mid \inputFeatures}$ to be strictly positive due to their subjective nature (see \cref{ass:positivePosterior}).
We repeat each experiment $5$ times to capture randomness of our measurements for both the model and adversary, and show average results across all $5$ repetitions.
At the start of the experiment, we split the data into train $\dataset_T$, evaluation $\dataset_E$, and auxiliary set $\dataset_A$. We train the model \CNN on the train set $\dataset_T$ for the chosen learning task and then estimate its utility on the evaluation set $\dataset_E$. 
We estimate the model's utility as its multiplicative gain in task accuracy $\smash{\gainEst(\labelY; \modelFeatures)}$ (see \cref{eq:util_gain}).
After model training and evaluation, we train both the label-only and features adversary on the auxiliary data $\dataset_A$. For a given sensitive attribute $\labelS$, we estimate the adversary's gain as $\gainEst(\labelS; \modelFeatures \mid \labelY)$ (see \cref{eq:adv_gain}).

%
\para{Learning Techniques} We implement two learning techniques: (1) standard ERM with SGD, and (2) attribute censoring to learn representations that hide a given sensitive attribute.
In our main experiment, we use the gradient reversal strategy (GRAD) introduced by \citet{RaffS18} for censoring with a censoring parameter of $100$. We choose GRAD because it effectively hides the chosen sensitive attribute without large drops in model performance~\citep{ZhaoCTG20}, and, unlike other censoring techniques, can be applied to any model architecture. In \cref{sec:add_exp}, we show equivalent results for another learning technique that aims to hide sensitive information, adversarial representation learning.
%
%

\subsection{The Potential Harms of Fundamental Leakage}
\label{subsec:emp_fundamental}
\begin{figure}
	\includegraphics[width=\columnwidth]{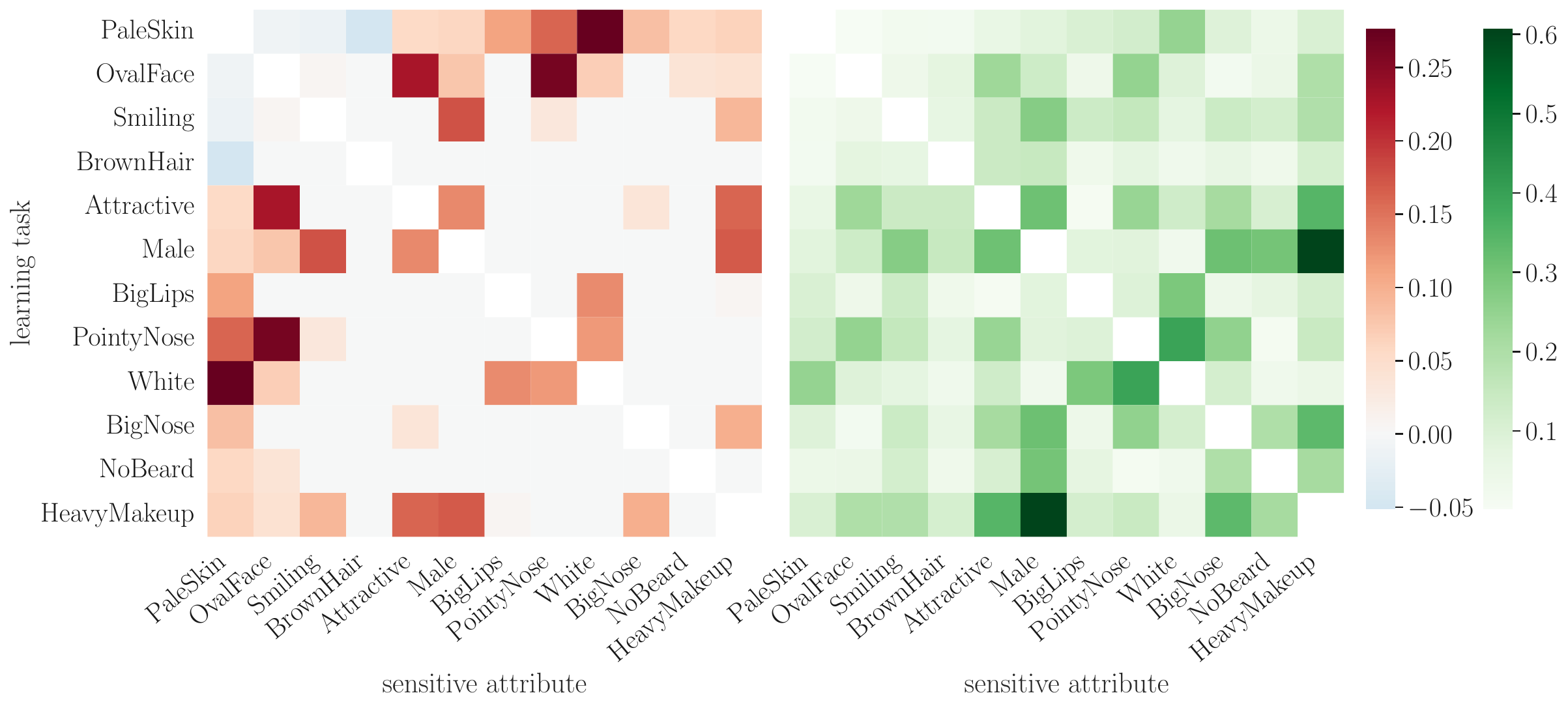}
	\caption{
    \textbf{Fundamental leakage: the task label reveals information about other data attributes, which might not be obvious to data subjects.} 
    Attribute inference gain of the label-only adversary\subfig{left} and pairwise Pearson's correlation between attributes\subfig{right}.
    In the LFWA+ dataset, the `Attractive' label is highly correlated with the perceived gender. Thus, predicting the `Attractive' label will reveal information about gender. 
    }
	\label{fig:corr_leakage}
    \vspace*{-5mm}
\end{figure}

\cref{fig:corr_leakage}\subfig{left} shows the adversary's gain $\smash{\gainEst(\labelS; \labelY)}$ in predicting sensitive attribute $\labelS$ from $\labelY$ for pair-wise combinations of $12$ learning tasks and sensitive attributes. This gain represents the fundamental leakage of a system that fulfils its intended purpose, i.e., that produces accurate task labels for the chosen learning task. The matrix in \cref{fig:corr_leakage}\subfig{right} shows the absolute pairwise Pearson's correlation coefficient between attribute labels.
\cref{fig:corr_leakage} shows that, as expected, a strong linear relationship between the learning task and the sensitive attribute targeted by the adversary leads to a large fundamental leakage. For instance, `Attractive' and `Male' are negatively correlated with a correlation coefficient of $-0.3094$. This is already enough to increase the adversary's gain in inferring sensitive attribute `Male' when the learning task is `Attractive'.
Although the increase is small in this case, it illustrates how inferences due to a task's fundamental leakage can not only be counterintuitive, but also reveal information that could lead to harm (in this case, discrimination due to gender).

The fundamental leakage for some pairs of tasks and sensitive attributes is highly concerning. It implies that users reveal to the service provider not only their task label but also \emph{any attribute that is correlated with the chosen task}. This consequence is rarely made explicit to users when they are informed about the data collection and processing, and might lead to unexpected harms beyond those associated with revealing the task label itself.
In the example above, a user expecting to only reveal `attractiveness' might not expect that their gender is revealed, with the ensuing risks of discrimination.
To better inform users about their risk, providers would have to list all sensitive attributes that might be leaked through a task's fundamental leakage. Knowing \emph{all} such attributes is infeasible. To address this problem, service providers could empirically evaluate whether attributes considered particularly sensitive are part of the fundamental leakage and inform data subjects about the result.
%
\subsection{The Least-Privilege And Utility Trade-Off}
\label{subsec:emp_tradeoff}
One way to interpret \cref{th:lpl_impossibility} is that whenever the features learned by a model are useful for a given prediction task, there always exists a sensitive attribute for which an adversary gains an advantage from observing a target record's feature representation.
We experimentally show this fundamental limit of least-privilege learning.

\begin{figure}[h!]
    \centering
    \includegraphics[width=\columnwidth]{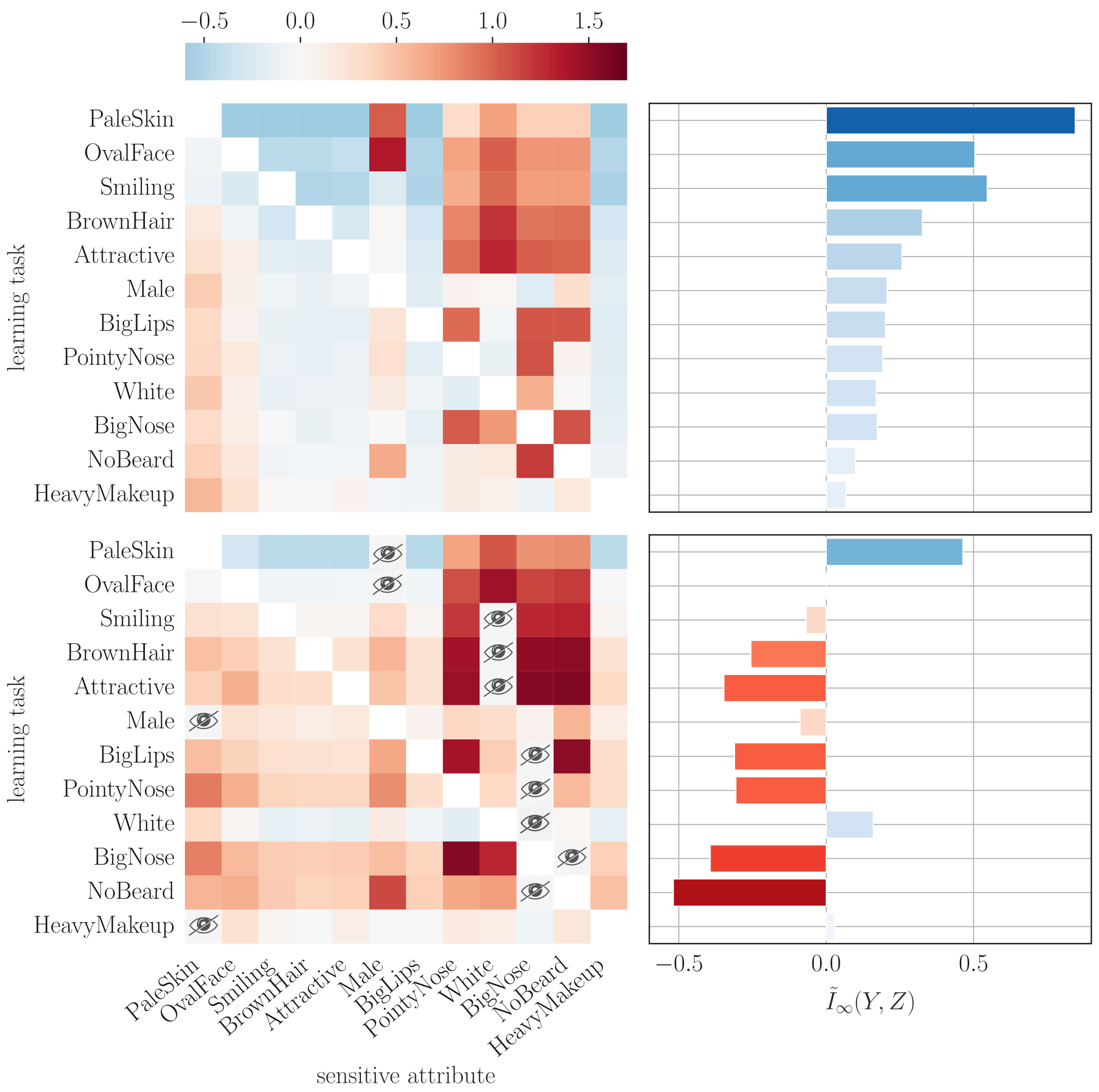}
    \caption{\textbf{If the model-generated representations have utility for the task\subfig{right}, there exists a sensitive attribute with an even higher inference gain for the adversary\subfig{left, \textcolor{red!80!black}{red} means more leakage}.} This holds for both standard ERM\subfig{top} and attribute censoring\subfig{bottom} where we censor the attribute with highest leakage in the respective ERM model (marked as $\textcolor{gray}{\tikzmarknode[strike out,draw]{2}{\text{\footnotesize \faEye}}}$). Censoring has a `whack-a-mole' effect: as we censor one attribute, leakage of another attribute increases.
    }
    \label{fig:put}
\end{figure}

\cref{fig:put} compares the trade-off between utility and attribute leakage of models trained with standard SGD \subfig{top} and with GRAD to censor the representation of a single sensitive attribute \subfig{bottom}. The blue bars in \cref{fig:put}\subfig{right} show the model's utility for learning task $\labelY$ measured as $\gainEst(\labelY; \modelFeatures)$. The heatmaps in \cref{fig:put}\subfig{left} show the difference between the adversary's inference gain and the model's utility:
$$\Delta_\mathsf{Adv} \smash{\define} \gainEst(\labelS; \modelFeatures \mid \labelY) - \gainEst(\labelY; \modelFeatures).$$
Each row corresponds to a learning task $\labelY$, and each column to a sensitive attribute targeted by the adversary.

\para{Across Learning Tasks} In \cref{fig:put}\subfig{top}, we see that in the LFWA+ dataset the features learned by a model that predicts attribute `Smiling' increase the adversary's inference gain for attribute `White'. Different tasks result in a high leakage for different attributes, e.g., `OvalFace' reveals a lot of additional information about gender, while `HeavyMakeup' is indicative of `PaleSkin' and `Attractive' but does not have much influence on the inference power of other attributes. Importantly, for \emph{every} task, there is at least one sensitive attribute  with $\Delta_\mathsf{Adv} > 0$.
These results confirm that the features learned by a model trained to perform well on its learning task reveal more information than a task's fundamental leakage and thus violate the LPP.
\cref{fig:adv_gain} in \cref{sec:add_exp} shows these results in terms of the adversary's absolute inference gain $\gainEst(\labelS; \modelFeatures \mid \labelY)$. It confirms that the adversary's absolute inference gain increases with a model's performance on its intended task.

\para{Across Learning Techniques} In \cref{fig:put}\subfig{bottom}, we show results for models trained under attribute censoring, a common technique used to address unintended information leakage~\citep{SongS19, BrownMM22, ZhaoCTG20}. For each task, we censor the attribute with the highest leakage under standard training. First, as expected, censoring limits the leakage of the censored attribute. However, we observe that the trade-off from \cref{th:lpl_impossibility} holds: In all our experiments, an adversary can find another data attribute with $\Delta_\mathsf{Adv} > 0$ that thus violates the LPP.  

In \cref{fig:maxent} in \cref{sec:add_exp}, we present experiments for another learning technique, adversarial representation learning, that has been shown to be an efficient censoring technique~\citep{ZhaoCTG20}. We find similar results: For two out of the four sensitive attributes we test, the adversary's inference gain exceeds the features' utility $\gainEst(S, Z \mid Y) > \gainEst(\labelY, \modelFeatures)$. These results confirm that the trade-off between a representation's utility and the LPP holds regardless of the learning technique used to obtain $\model_E(\inputFeatures)$.

\para{Across Hidden Layers} We conduct additional experiments that show that the strict trade-off of \cref{th:lpl_impossibility} applies regardless of the model architecture (see \cref{fig:resnet} in \cref{sec:add_exp}) or the hidden layer at which the adversary observes a record's representation (see \cref{fig:layers} in \cref{sec:add_exp}). These results are in line with prior work by \citet{MoBMH21} and \citet{MelisSDS19}. Even though higher layers are expected to be more learning task specific, an adversary can misuse the shared representations at any layer to predict attributes other than the intended task. We show that this leakage even exceeds the fundamental leakage of the task label itself and violates the LPP.

\para{Across Datasets and Model Architectures} Besides images, tabular data is another type of data for which representation sharing is commonly suggested as a technique to limit unintended information leakage~\citep{ZhaoCTG20, ArikP21, CreagerMJ19, RaffS18}. In \cref{fig:maxent,fig:tabnet} in \cref{sec:add_exp}, we show that the trade-off also holds for this type of data (and the associated model architectures) using two tabular datasets: it is not possible to find feature representations that have high utility for a prediction task but fulfil the LPP.

\para{In Collaborative Learning Settings} The gradient updates shared with a ML service provider in the collaborative learning setting are a noisy version of the feature representations learned by the model \citep{MelisSDS19}. In \cref{fig:gradients} in \cref{sec:add_exp}, we show that hence the fundamental trade-off of \cref{th:lpl_impossibility} equally applies to gradient updates shared for model training. Our experiments thus show that LPL cannot, as envisioned by \citet{MelisSDS19}, provide the promised solution to address unintended feature leakage in the collaborative learning setting.

\para{Takeaways} Our experiments confirm \emph{the fundamental trade-off} we derive in \cref{th:lpl_impossibility}: the representations generated by models that perform well on their intended task fail to fulfil the LPP. Censoring techniques can be used to limit the adversary's inference gain on a particular attribute but cannot avoid the trade-off: there always exists an attribute, other than the learning task, that violates the LPP. This `whack-a-mole' effect is a phenomenon observed in related scenarios, such as, privacy-preserving data publishing~\citep{NarayananS19}.
Our experiments are the first to study the trade-off between a representation's utility and unintended information leakage across a large combination of learning tasks and sensitive attributes. In contrast to prior work that evaluates unintended feature leakage for a single learning task and fixed sensitive attribute, this enables us to show that although it is sometimes possible to restrict information leakage for a single attribute, the learned representations do not fulfil the LPP.

\section{Limitations and Future Work}
Our results hold under a mild assumption on the data distribution (\cref{ass:positivePosterior}) and for a specific operational measure of leakage (\cref{subsec:leakage_measure}). By relaxing these two aspects, it might be possible to achieve a more favourable trade-off than what we find.

First, the trade-off could possibly be less stringent for specific data distributions for which \cref{ass:positivePosterior} does not hold. In such cases, it might be possible that for some learning tasks, there exists a representation that provides some utility gain and at the same time constrains the maximal leakage, i.e., that simultaneously fulfils the two conditions stated in \cref{th:lpl_impossibility}. It could be a direction for future work to formally analyse whether such representations are achievable for distributions that do not satisfy \cref{ass:positivePosterior}.

Second, the leakage measure used to formalise the LPP could be relaxed by restricting the set of possible inference targets to a less rich class of distributions. For instance, the class of attributes targeted by the adversary could be limited to those “efficiently identifiable” \citep{HebertKRR18}. However, such modifications would require a different analytical toolbox and deviate from the promise of LPL to leak “nothing else” \citep{BrownMM22} but the intended task. In addition, although the leakage measure considered here might seem strong, it is already significantly weaker than standard privacy notions, such as LDP, as it measures attribute inference success in expectation over the population as opposed to worst-case risks.

\section{Conclusions}
\label{sec:conclusions}

The promise of least-privilege learning --- to learn feature representations that are useful for a given task but avoid leakage of any information that might be misused and cause harm --- is extremely appealing. 
In this paper, we show that in certain realistic settings where there exists inherent uncertainty about the task labels, any representation that provides utility for its intended task must always leak information about attributes other than the task and thus does neither fulfil the least-privilege principle nor provide a better utility-privacy trade-off than existing formal privacy-preserving techniques. 
Furthermore, we show that, due to a task's fundamental leakage, even representations that fulfil the least-privilege principle might not prevent all the potential harms. 
These issues apply to any setting in which users share representations instead of raw data records and regardless of the learning technique used to learn the feature mapping that produce these representations.

\section*{Impact Statement}
This paper validates claims from prior works that the concept of least-privilege learning might provide a promising avenue to mitigate harms of potential misuse of machine learning. Prior works repeatedly suggest that least-privilege learning might achieve a better trade-off between revealing information about an intended task and unintended information leakage than existing techniques, such as differential privacy.
In this paper, we formally prove that least-privilege learning might not fulfil this promise: We show that there is an inherent trade-off between a representation's utility and its unintended information leakage that cannot be overcome by any learning technique.
The results of this paper will help practitioners 
to evaluate the potential benefits and shortcomings of a newly proposed technique. Our formalisation of a previously only informally stated property will enable researchers to formally compare the trade-offs of alternative methods to prevent data misuse in machine learning.

\clearpage

\section*{Acknowledgements}
This work is partially funded by the Swiss National Science Foundation under grant 200021-188824. NP acknowledges support from the Canada CIFAR AI Chair. We would like to thank Christian Knabenhans for providing a basis for the implementation of the empirical evaluation of this work.

\bibliography{least_privilege}

\begin{thebibliography}{40}
\providecommand{\natexlab}[1]{#1}
\providecommand{\url}[1]{\texttt{#1}}
\expandafter\ifx\csname urlstyle\endcsname\relax
  \providecommand{\doi}[1]{doi: #1}\else
  \providecommand{\doi}{doi: \begingroup \urlstyle{rm}\Url}\fi

\bibitem[Alemi et~al.(2016)Alemi, Fischer, Dillon, and Murphy]{AlemiFDM16}
Alemi, A.~A., Fischer, I., Dillon, J.~V., and Murphy, K.
\newblock Deep variational information bottleneck.
\newblock In \emph{International Conference on Learning Representations}, 2016.

\bibitem[Arik \& Pfister(2021)Arik and Pfister]{ArikP21}
Arik, S.~{\"O}. and Pfister, T.
\newblock Tabnet: Attentive interpretable tabular learning.
\newblock \emph{Proceedings of the AAAI Conference on Artificial Intelligence},
  2021.

\bibitem[Arimoto(1977)]{Arimoto77}
Arimoto, S.
\newblock Information measures and capacity of order $\alpha$ for discrete
  memoryless channels.
\newblock \emph{Topics in information theory}, 1977.

\bibitem[Asoodeh \& Calmon(2020)Asoodeh and Calmon]{AsoodehC20}
Asoodeh, S. and Calmon, F.~P.
\newblock Bottleneck problems: An information and estimation-theoretic view.
\newblock \emph{Entropy}, 22\penalty0 (11):\penalty0 1325, 2020.

\bibitem[Boenisch et~al.(2023)Boenisch, Dziedzic, Schuster, Shamsabadi,
  Shumailov, and Papernot]{BoenischDS23}
Boenisch, F., Dziedzic, A., Schuster, R., Shamsabadi, A.~S., Shumailov, I., and
  Papernot, N.
\newblock Reconstructing individual data points in federated learning hardened
  with differential privacy and secure aggregation.
\newblock In \emph{{IEEE} {S}\&{P}}, 2023.

\bibitem[Brown et~al.(2022)Brown, Martinez-del Rincon, and Miller]{BrownMM22}
Brown, G., Martinez-del Rincon, J., and Miller, P.
\newblock Least privilege learning for attribute obfuscation.
\newblock In \emph{Pattern Recognition}, 2022.

\bibitem[Chatzikokolakis et~al.(2023)Chatzikokolakis, Cherubin, Palamidessi,
  and Troncoso]{ChatzikokolakisCPT23}
Chatzikokolakis, K., Cherubin, G., Palamidessi, C., and Troncoso, C.
\newblock Bayes security: {A} not so average metric.
\newblock In \emph{36th {IEEE} Computer Security Foundations Symposium, {CSF}
  2023, Dubrovnik, Croatia, July 10-14, 2023}, pp.\  388--406. {IEEE}, 2023.

\bibitem[Chi et~al.(2018)Chi, Owusu, Yin, Yu, Chan, Tague, and Tian]{ChiOY18}
Chi, J., Owusu, E., Yin, X., Yu, T., Chan, W., Tague, P., and Tian, Y.
\newblock Privacy partitioning: Protecting user data during the deep learning
  inference phase.
\newblock \emph{arXiv preprint arXiv:1812.02863}, 2018.

\bibitem[Citron \& Solove(2022)Citron and Solove]{CitronS22}
Citron, D.~K. and Solove, D.~J.
\newblock Privacy harms.
\newblock \emph{BUL Rev.}, 102:\penalty0 793, 2022.

\bibitem[Creager et~al.(2019)Creager, Madras, Jacobsen, Weis, Swersky, Pitassi,
  and Zemel]{CreagerMJ19}
Creager, E., Madras, D., Jacobsen, J.-H., Weis, M., Swersky, K., Pitassi, T.,
  and Zemel, R.
\newblock Flexibly fair representation learning by disentanglement.
\newblock In \emph{ICML}, 2019.

\bibitem[Dwork et~al.(2014)Dwork, Roth, et~al.]{DworkR14}
Dwork, C., Roth, A., et~al.
\newblock The algorithmic foundations of differential privacy.
\newblock \emph{Foundations and Trends{\textregistered} in Theoretical Computer
  Science}, 9\penalty0 (3--4):\penalty0 211--407, 2014.

\bibitem[{European Parliament and Council of the European Union}(2016)]{gdpr}
{European Parliament and Council of the European Union}.
\newblock Regulation (eu) 2016/679 of the european parliament and of the
  council of 27 april 2016 on the protection of natural persons with regard to
  the processing of personal data and on the free movement of such data, and
  repealing directive 95/46/ec (general data protection regulation).
\newblock
  \url{https://eur-lex.europa.eu/legal-content/EN/TXT/HTML/?uri=CELEX:32016R0679&from=EN},
  2016.

\bibitem[Fischer(2020)]{Fisher20}
Fischer, I.
\newblock The conditional entropy bottleneck.
\newblock \emph{Entropy}, 22\penalty0 (9):\penalty0 999, 2020.

\bibitem[Ganju et~al.(2018)Ganju, Wang, Yang, Gunter, and Borisov]{GanjuWY18}
Ganju, K., Wang, Q., Yang, W., Gunter, C.~A., and Borisov, N.
\newblock Property inference attacks on fully connected neural networks using
  permutation invariant representations.
\newblock In \emph{ACM SIGSAC Conference on Computer and Communications
  Security}, 2018.

\bibitem[Hao(2019)]{FedApple}
Hao, K.
\newblock How {Apple} personalizes {Siri} without hoovering up your data.
\newblock
  \url{https://www.technologyreview.com/2019/12/11/131629/apple-ai-personalizes-siri-federated-learning/},
  2019.

\bibitem[He et~al.(2015)He, Zhang, Ren, and Sun]{HeZR15}
He, K., Zhang, X., Ren, S., and Sun, J.
\newblock Deep residual learning for image recognition.
\newblock \emph{arXiv}, 2015.

\bibitem[H{\'e}bert-Johnson et~al.(2018)H{\'e}bert-Johnson, Kim, Reingold, and
  Rothblum]{HebertKRR18}
H{\'e}bert-Johnson, U., Kim, M., Reingold, O., and Rothblum, G.
\newblock Multicalibration: Calibration for the (computationally-identifiable)
  masses.
\newblock In \emph{ICML}, 2018.

\bibitem[Huang et~al.(2008)Huang, Mattar, Berg, and Learned-Miller]{LFWA}
Huang, G.~B., Mattar, M., Berg, T., and Learned-Miller, E.
\newblock Labeled faces in the wild: A database forstudying face recognition in
  unconstrained environments.
\newblock In \emph{Workshop on faces in'Real-Life'Images: detection, alignment,
  and recognition}, 2008.

\bibitem[Issa et~al.(2019)Issa, Wagner, and Kamath]{IssaWK19}
Issa, I., Wagner, A.~B., and Kamath, S.
\newblock An operational approach to information leakage.
\newblock \emph{IEEE Transactions on Information Theory}, 66\penalty0
  (3):\penalty0 1625--1657, 2019.

\bibitem[Kohavi \& Becker(2013)Kohavi and Becker]{Adult}
Kohavi, R. and Becker, B.
\newblock {UCI} {M}achine {L}earning {R}epository.
\newblock \url{https://archive.ics.uci.edu/ml/datasets/Adult}, 2013.

\bibitem[Liao et~al.(2019)Liao, Kosut, Sankar, and du~Pin~Calmon]{LiaoKSC19}
Liao, J., Kosut, O., Sankar, L., and du~Pin~Calmon, F.
\newblock Tunable measures for information leakage and applications to
  privacy-utility tradeoffs.
\newblock \emph{IEEE Transactions on Information Theory}, 2019.

\bibitem[Maeng et~al.(2024)Maeng, Guo, Kariyappa, and Suh]{MaengGKS24}
Maeng, K., Guo, C., Kariyappa, S., and Suh, G.~E.
\newblock Bounding the invertibility of privacy-preserving instance encoding
  using fisher information.
\newblock In \emph{Neur{IPS}}, 2024.

\bibitem[Makhdoumi et~al.(2014)Makhdoumi, Salamatian, Fawaz, and
  M{\'e}dard]{MakhdoumiSFM14}
Makhdoumi, A., Salamatian, S., Fawaz, N., and M{\'e}dard, M.
\newblock From the information bottleneck to the privacy funnel.
\newblock In \emph{2014 IEEE Information Theory Workshop (ITW 2014)}, pp.\
  501--505. IEEE, 2014.

\bibitem[McMahan \& Ramage(2017)McMahan and Ramage]{FedGoogle}
McMahan, B. and Ramage, D.
\newblock Federated learning: Collaborative machine learning without
  centralized training data.
\newblock
  \url{http://ai.googleblog.com/2017/04/federated-learning-collaborative.html},
  2017.

\bibitem[Melis et~al.(2019)Melis, Song, De~Cristofaro, and
  Shmatikov]{MelisSDS19}
Melis, L., Song, C., De~Cristofaro, E., and Shmatikov, V.
\newblock Exploiting {Unintended} {Feature} {Leakage} in {Collaborative}
  {Learning}.
\newblock In \emph{SP}, 2019.

\bibitem[Mireshghallah et~al.(2021)Mireshghallah, Taram, Jalali, Elthakeb,
  Tullsen, and Esmaeilzadeh]{MireshghallahTJE21}
Mireshghallah, F., Taram, M., Jalali, A., Elthakeb, A. T.~T., Tullsen, D., and
  Esmaeilzadeh, H.
\newblock Not all features are equal: Discovering essential features for
  preserving prediction privacy.
\newblock In \emph{Proceedings of the Web Conference 2021}, 2021.

\bibitem[Mo et~al.(2021)Mo, Borovykh, Malekzadeh, Haddadi, and
  Demetriou]{MoBMH21}
Mo, F., Borovykh, A., Malekzadeh, M., Haddadi, H., and Demetriou, S.
\newblock Layer-wise {Characterization} of {Latent} {Information} {Leakage} in
  {Federated} {Learning}.
\newblock \emph{ar{X}iv:2010.08762}, 2021.

\bibitem[Narayanan \& Shmatikov(2019)Narayanan and Shmatikov]{NarayananS19}
Narayanan, A. and Shmatikov, V.
\newblock Robust de-anonymization of large sparse datasets: a decade later.
\newblock
  \url{https://www.cs.princeton.edu/~arvindn/publications/de-anonymization-retrospective.pdf},
  2019.

\bibitem[Nissenbaum(2004)]{Nissenbaum04}
Nissenbaum, H.
\newblock Privacy as contextual integrity.
\newblock \emph{Wash. L. Rev.}, 79:\penalty0 119, 2004.

\bibitem[Osia et~al.(2018)Osia, Taheri, Shamsabadi, Katevas, Haddadi, and
  Rabiee]{OsiaTS18}
Osia, S.~A., Taheri, A., Shamsabadi, A.~S., Katevas, K., Haddadi, H., and
  Rabiee, H.~R.
\newblock Deep private-feature extraction.
\newblock \emph{IEEE Transactions on Knowledge and Data Engineering}, 2018.

\bibitem[Raff \& Sylvester(2018)Raff and Sylvester]{RaffS18}
Raff, E. and Sylvester, J.
\newblock Gradient {Reversal} {Against} {Discrimination}.
\newblock In \emph{FACCT}, 2018.

\bibitem[Raji et~al.(2021)Raji, Denton, Bender, Hanna, and
  Paullada]{RajiBPDH21}
Raji, I.~D., Denton, E., Bender, E.~M., Hanna, A., and Paullada, A.
\newblock Ai and the everything in the whole wide world benchmark.
\newblock In \emph{NeurIPS Datasets and Benchmarks Track}, 2021.

\bibitem[Sablayrolles et~al.(2019)Sablayrolles, Douze, Schmid, Ollivier, and
  J{\'{e}}gou]{SablayrollesDSO19}
Sablayrolles, A., Douze, M., Schmid, C., Ollivier, Y., and J{\'{e}}gou, H.
\newblock White-box vs black-box: Bayes optimal strategies for membership
  inference.
\newblock In \emph{Proceedings of the 36th International Conference on Machine
  Learning, {ICML} 2019, 9-15 June 2019, Long Beach, California, {USA}},
  Proceedings of Machine Learning Research. {PMLR}, 2019.

\bibitem[Saltzer \& Schroeder(1975)Saltzer and Schroeder]{SaltzerS75}
Saltzer, J.~H. and Schroeder, M.~D.
\newblock The protection of information in computer systems.
\newblock \emph{Proceedings of the IEEE}, 1975.

\bibitem[Song \& Shmatikov(2019)Song and Shmatikov]{SongS19}
Song, C. and Shmatikov, V.
\newblock Overlearning {Reveals} {Sensitive} {Attributes}.
\newblock In \emph{NIPS}, 2019.

\bibitem[Song et~al.(2022)Song, Kim, Park, Shin, and Lee]{SongKPSL22}
Song, H., Kim, M., Park, D., Shin, Y., and Lee, J.-G.
\newblock Learning from noisy labels with deep neural networks: A survey.
\newblock \emph{IEEE Transactions on Neural Networks and Learning Systems},
  2022.

\bibitem[{Texas Department of State Health Services, Austin,
  Texas}(2013)]{Texas}
{Texas Department of State Health Services, Austin, Texas}.
\newblock {Texas Hospital Inpatient Discharge Public Use Data File 2013 Q1-Q4}.
\newblock \url{https://www.dshs.texas.gov/THCIC/Hospitals/Download.shtm}, 2013.
\newblock Accessed 2020-06-01.

\bibitem[Tishby et~al.(2000)Tishby, Pereira, and Bialek]{TishbyPB00}
Tishby, N., Pereira, F.~C., and Bialek, W.
\newblock The information bottleneck method.
\newblock \emph{arXiv preprint physics/0004057}, 2000.

\bibitem[Wang et~al.(2018)Wang, Zhang, Bao, Zhu, Cao, and Yu]{WangZB18}
Wang, J., Zhang, J., Bao, W., Zhu, X., Cao, B., and Yu, P.~S.
\newblock Not just privacy: Improving performance of private deep learning in
  mobile cloud.
\newblock In \emph{ACM SIGKDD International Conference on Knowledge Discovery
  \& Data Mining}, 2018.

\bibitem[Zhao et~al.(2020)Zhao, Chi, Tian, and Gordon]{ZhaoCTG20}
Zhao, H., Chi, J., Tian, Y., and Gordon, G.~J.
\newblock Trade-offs and guarantees of adversarial representation learning for
  information obfuscation.
\newblock \emph{Neur{IPS}}, 2020.

\end{thebibliography}
\bibliographystyle{resource/icml2024}

\clearpage
\appendix
\onecolumn
\renewcommand{\figurewidth}{.9\linewidth}

\section{Formal Details}
\label{sec:proofs}

\para{Maximal Leakage}
First, let us provide some useful properties of maximal leakage.

Maximal leakage and conditional maximal leakage are defined~\citep{IssaWK19} as follows:
\begin{align}
    \label{eq:max_leakage_full}
   \maxLeakage(\inputFeatures \rightarrow \modelFeatures) \define \sup_{\labelS:~\labelS -
   \inputFeatures - \modelFeatures} \gain(\labelS;
   \modelFeatures), & &
    \maxLeakage(\inputFeatures \rightarrow \modelFeatures \mid \labelY) \define
   \sup_{\labelS:~\labelS - (\inputFeatures, \labelY) - \modelFeatures} \gain(\labelS; \modelFeatures \mid \labelY).
\end{align}

\begin{lemma}[\citet{IssaWK19}]
    \label{th:max_leakage_props}
    Maximal leakage has the following properties:
    \begin{enumerate}
        \item Maximal leakage bounds mutual information for $\alpha \in \{1, \infty\}$: $\maxLeakage(\inputFeatures \rightarrow
            \modelFeatures) \geq \utilInfo(\inputFeatures; \modelFeatures)$
        \item Maximal leakage satisfies data processing inequalities in a Markov chain $Y - X -Z$:
            \begin{align*}
                \maxLeakage(Y \rightarrow Z) \leq \min\{ \maxLeakage(X \rightarrow Z), \maxLeakage(Y \rightarrow X) \}
            \end{align*}
        \item Maximal leakage and conditional maximal leakage have the following closed forms:
            \begin{align*}
                \maxLeakage(\inputFeatures \rightarrow \modelFeatures) &= \log\left(\sum_{\modelFeaturesVal \in \modelFeaturesSpace} \max_{\inputFeaturesVal \in \inputSpace} \dataDist(\modelFeaturesVal \mid \inputFeaturesVal)\right) \\
                        \maxLeakage(\inputFeatures \rightarrow \modelFeatures \mid \labelY) &=
                \log\left( \max_{\labelYVal \in \labelYSpace} \sum_{\modelFeaturesVal \in
            \modelFeaturesSpace}
            \max_{\inputFeaturesVal \in \supp(\inputFeatures \mid \labelY = \labelYVal)} \dataDist(\modelFeaturesVal \mid
        \inputFeaturesVal, \labelYVal) \right)
            \end{align*}
    \end{enumerate}
\end{lemma}

\para{Maximally Revealing Attributes}
We demonstrate that there exists an attribute that is maximally revealing, i.e., achieves
the supremum $\sup_{\labelS:~\labelS - \inputFeatures - \modelFeatures} \gain(\labelS;
\modelFeatures)$, and show how to construct this attribute.

\begin{definition}\label{def:max-revealing-attr}
    For a given $P_X$, a \emph{maximally revealing attribute} $S^*$~\citep[see ``shattering distribution'',][]{IssaWK19}
    is defined over $\labelSSpace \define \inputSpace \times \mathbb{Z}$ via the following conditional probability distribution:
    \begin{equation}
        P_{S^* \mid X}((x', k) \mid x) \define \begin{cases}
            \frac{p_{\min}}{P_X(x)},& x' = x \text{ and } 1 \leq k \leq \lfloor r(x) \rfloor \\ 
            1 - \frac{(\lceil r(x) \rceil - 1) \cdot p_{\min}}{P_X(x)},& x' = x \text{ and } k = \lceil r(x) \rceil \\
            0, & x' \neq x,
        \end{cases}
    \end{equation}
    for any $x, x' \in \inputSpace$, and $k \in [1, \ldots, \lceil r(x) \rceil]$, where $p_{\min} \define \min_{x \in \inputSpace} P_X(x)$ and $r(x) \define \nicefrac{P_X(x)}{p_{\min}}$.
\end{definition}

\begin{proposition}[\citet{IssaWK19}]
    \label{th:max_revealing_class}
    For any $Z = f_E(X)$ and $P_X$, the maximally revealing attribute $S^*$ leads to the highest inference gain:
    \begin{align}
        \sup_{\labelS:~\labelS - \inputFeatures - \modelFeatures} \gain(\labelS; \modelFeatures) =
       \gain(\labelS^*; \modelFeatures)
    \end{align}
\end{proposition}
\begin{proof}
    It is sufficient to show that for any $P_{Z \mid X}$ and $P_X$, we have $\maxLeakage(X \rightarrow Z) = \gain(S^*; Z)$. By \cref{th:max_leakage_props} and the definition of $\gain$, this is equivalent to showing:
   \begin{equation}
        \label{eq:proof-def-equivalence}
        \log\left(\sum_{\modelFeaturesVal \in \modelFeaturesSpace} \max_{\inputFeaturesVal \in
        \inputSpace} \dataDist_{Z \mid X}(\modelFeaturesVal \mid \inputFeaturesVal)\right) = \log\left(\frac{\sum_{z \in \modelFeaturesSpace} \max_{s \in \labelSSpace} P_{S^*, Z}(s, z)}{\max_{s \in \labelSSpace} P_{S^*}(s) } \right).
   \end{equation}
   We can see that $\max_{s \in \labelSSpace} P_{S^* \mid Z}(s, z)$ from the RHS has the following form:
   \begin{align}
       \max_{(x', k) \in \labelSSpace} P_{S^*, Z}((x', k), z) &= \max_{(x', k) \in \labelSSpace} P_{S^* \mid X}((x', k) \mid x') \cdot P_X(x') \cdot P_{Z \mid X}(z \mid x') \\
       &= \max_{(x', k) \in \labelSSpace \text{ s.t. } k = 1} P_{S^* \mid X}((x', 1) \mid x') \cdot P_X(x') \cdot P_{Z \mid X}(z \mid x') \\
       &= \max_{x \in \inputSpace} \frac{p_{\min}}{P_X(x)} \cdot P_X(x) \cdot P_{Z \mid X}(z \mid x) = p_{\min} \cdot \max_{x \in \inputSpace} P_{Z \mid X}(z \mid x),
   \end{align}
   where the first equality is by Markov chain $S^* - X - Z$, and the second and third are by \cref{def:max-revealing-attr}. 
   Combining the resulting expression and the fact that by \cref{def:max-revealing-attr} we have
   \begin{align}
       \max_{s \in \labelSSpace} P_{S^*}(s) = \max_{s \in \labelSSpace} \sum_{x \in \inputSpace} P_{S^* \mid X}(s) \cdot P_X(x) = p_{\min},
    \end{align}
    we get the sought result.
\end{proof}

We also show the following technical property of a maximally revealing attribute.
\begin{proposition}
    \label{th:max_revealing_and_positivity}
    Suppose that $|\inputSpace| > 1$. 
    The conditional distribution $\dataDist_{\labelS^* \mid \inputFeatures}$ is non-positive: there
    exist a $\labelSVal \in \labelSSpace$ and $\inputFeaturesVal \in \inputSpace$ such that
    $\dataDist_{\labelS^* \mid \inputFeatures}(\labelSVal \mid \inputFeaturesVal) = 0$.

    \proof{By \cref{def:max-revealing-attr}, for any $x \in \inputSpace$ it is sufficient to take any $(x', k) \in \labelSSpace$ where $x' \neq x$.}
\end{proposition}

\subsection{Omitted Proofs} Next, we provide the proofs of the formal statements in the main body of the
paper.

\begin{proof}[Proof of \cref{th:ulpl_impossibility}]
    By construction, we have a Markov chain $\labelY - \inputFeatures - \modelFeatures$. 
    Observe that $\labelS^* \neq \labelY$ by
    \cref{th:max_revealing_and_positivity} and strict positivity
            (\cref{ass:positivePosterior}). Therefore,
    \begin{align}
        \sup_{\labelS:~\labelS - \inputFeatures - \modelFeatures,~\labelS \neq \labelY} \gain(\labelS; \modelFeatures)
        = \gain(\labelS^*; \modelFeatures)
        = \sup_{\labelS:~\labelS - \inputFeatures - \modelFeatures} \gain(\labelS; \modelFeatures)
        = \maxLeakage(\inputFeatures \rightarrow \modelFeatures)
    \end{align}
    by \cref{th:max_revealing_class} and the definition of maximal leakage. Finally, applying the properties in \cref{th:max_leakage_props}, we have $\utilInfo(\labelY; \modelFeatures) \leq
    \maxLeakage(\labelY \rightarrow \modelFeatures) \leq \maxLeakage(\inputFeatures \rightarrow
    \modelFeatures) = \sup_{\labelS:~\labelS - \inputFeatures - \modelFeatures,~\labelS \neq \labelY} \gain(\labelS; \modelFeatures) \leq \gamma$.
\end{proof}

\begin{remark}
\cref{th:ulpl_impossibility} only requires that $Y \neq S^*$. This holds under weaker assumptions than the strict positivity of the posterior (\cref{ass:positivePosterior}). For instance, it is sufficient that there exist any $\inputFeaturesVal, \inputFeaturesVal', \labelYVal$ such that both $\dataDist_{\inputFeatures \mid \labelY}(\inputFeaturesVal \mid \labelYVal) > 0$ and $\dataDist_{\inputFeatures \mid \labelY}(\inputFeaturesVal' \mid \labelYVal) > 0$.
\end{remark}

\begin{proof}[Proof of \cref{stmt:perfect-lpp}]
     For $\gamma = 0$, by \cref{th:max_leakage_props}  we have that $\maxLeakage(\inputFeatures \rightarrow \modelFeatures \mid \labelY) = I(\inputFeatures; \modelFeatures \mid \labelY) = 0$. As conditional mutual information is consistent with conditional independence, this holds if and only if we have Markov chain $\inputFeatures - \labelY - \modelFeatures$. 
\end{proof}


\begin{proof}[Proof of \cref{stmt:det-labels}]
If $Y=g(X),$ where $g(\cdot)$ is a deterministic function, we can choose $Z = f(Y)=f(g(X)).$
Evidently, this satisfies the Markov relationship $f(g(X)) - g(X) - X,$ {\it i.e.,} $Z-Y-X.$ (Note that $f(\cdot)$ may even be a randomized function for this to hold.) But this Markov relation implies $\maxLeakage(\inputFeatures \rightarrow \modelFeatures \mid \labelY)=0.$\end{proof}

\cref{stmt:positive-posterior-perfect-lpp} follows as a consequence of the conditional independence in \cref{stmt:perfect-lpp} and the fact that we have $\labelY - \inputFeatures - \modelFeatures$ by construction.

\begin{proof}[Proof of \cref{th:lpl_impossibility}]
    First, we show that in our Markov chain and under the strictly positive posterior assumption, we
    have a surprising result that $\maxLeakage(\inputFeatures \rightarrow \modelFeatures \mid
    \labelY) = \maxLeakage(\inputFeatures \rightarrow \modelFeatures)$.


    To see this, observe that maximal leakage has the following closed form by
    \cref{th:max_leakage_props}:
    \begin{align}
        \maxLeakage(\inputFeatures \rightarrow \modelFeatures \mid \labelY) &=
        \log\left( \max_{\labelYVal \in \labelYSpace)} \sum_{\modelFeaturesVal \in
            \modelFeaturesSpace}
            \max_{\inputFeaturesVal \in \supp(\inputFeatures \mid \labelY = \labelYVal)} \dataDist(\modelFeaturesVal \mid
        \inputFeaturesVal, \labelYVal) \right) \\
        & = \log\left( \max_{\labelYVal \in \labelYSpace} \sum_{\modelFeaturesVal \in
            \modelFeaturesSpace}
            \max_{\inputFeaturesVal \in \supp(\inputFeatures \mid \labelY = \labelYVal)} \dataDist(\modelFeaturesVal \mid
        \inputFeaturesVal) \right),
    \end{align}
    where the second equality is by the Markov chain $\labelY - \inputFeatures - \modelFeatures$.

    Next, observe that we have $\dataDist(\inputFeaturesVal \mid \labelYVal) \propto \dataDist(\labelYVal \mid
    \inputFeaturesVal)
    \cdot \dataDist(\inputFeaturesVal)$ which, by assumption, is positive
    so long as $\inputFeaturesVal \in \supp(\inputFeatures)$. As a
    consequence, the support of $\inputFeatures$ is independent of $\labelY$:
    $\supp(\inputFeatures \mid \labelY = \labelYVal) = \supp(\inputFeatures)$, for
    any $\labelYVal \in \labelYSpace$.
    Therefore, we can simplify the last form:
    \begin{align}\label{eq:proof-cond-uncond}
        \maxLeakage(\inputFeatures \rightarrow \modelFeatures \mid \labelY) &=
        \log\left( \max_{\labelYVal \in \labelYSpace} \sum_{\modelFeaturesVal \in
            \modelFeaturesSpace}
            \max_{\inputFeaturesVal \in \supp(\inputFeatures \mid \labelY = \labelYVal)} \dataDist(\modelFeaturesVal \mid
        \inputFeaturesVal) \right) \\
        & = \log\left(\sum_{\modelFeaturesVal \in \modelFeaturesSpace} \max_{\inputFeaturesVal \in
        \inputSpace} \dataDist(\modelFeaturesVal \mid \inputFeaturesVal)\right),
    \end{align}
    which is an equivalent form of $\maxLeakage(\inputFeatures \rightarrow
    \modelFeatures)$.

    Finally, to obtain the trade-off, it suffices to observe that $\utilInfo(\labelY; \modelFeatures) \leq \maxLeakage(\labelY \rightarrow \modelFeatures) \leq
    \maxLeakage(\inputFeatures \rightarrow \modelFeatures) \leq \gamma$, where the the first and the second inequalities are by properties of maximal leakage in \cref{th:max_leakage_props}.
\end{proof}

One way to interpret this result is that the maximally revealing attribute $S^*$ is not sensitive to conditioning on $\labelY$ so long as
 we have strict positivity of the posterior (\cref{ass:positivePosterior}).

\begin{proof}[Proof of \cref{th:ldp-lpp}]
    \citet{IssaWK19} show that if the map $\model_E(\cdot)$ is $\varepsilon$-LDP, then $\maxLeakage(\inputFeatures \rightarrow \modelFeatures) \leq \varepsilon$, where $\modelFeatures = \model_E(\inputFeatures)$. From the form in  \cref{eq:proof-cond-uncond}, we have that $\maxLeakage(\inputFeatures \rightarrow \modelFeatures \mid \labelY) \leq \maxLeakage(\inputFeatures \rightarrow \modelFeatures) \leq \varepsilon$.
\end{proof}



\section{Empirical Evaluation of the Utility and LPP Trade-off}
\label{sec:add_exp}
In this section, we provide additional details, as well as, additional empirical results that show that the strict trade-off between utility and the LPP  hold for any feature representation regardless of the feature learning technique, model architecture, or dataset.

\subsection{Details of Setup in Section~\ref{sec:empirical}}
Following \citet{MelisSDS19}, we use a convolutional network with three spatial convolution layers with 32, 64, and 128 filters, kernel size set to (3, 3), max pooling layers with pooling size set to 2, followed by two fully connected layers of size 256 and 2. We use ReLU as the activation function for all layers. 

\subsection{Adversarial Inference Gain}

\begin{figure}[h!]
    \centering
    \includegraphics[width=\figurewidth]{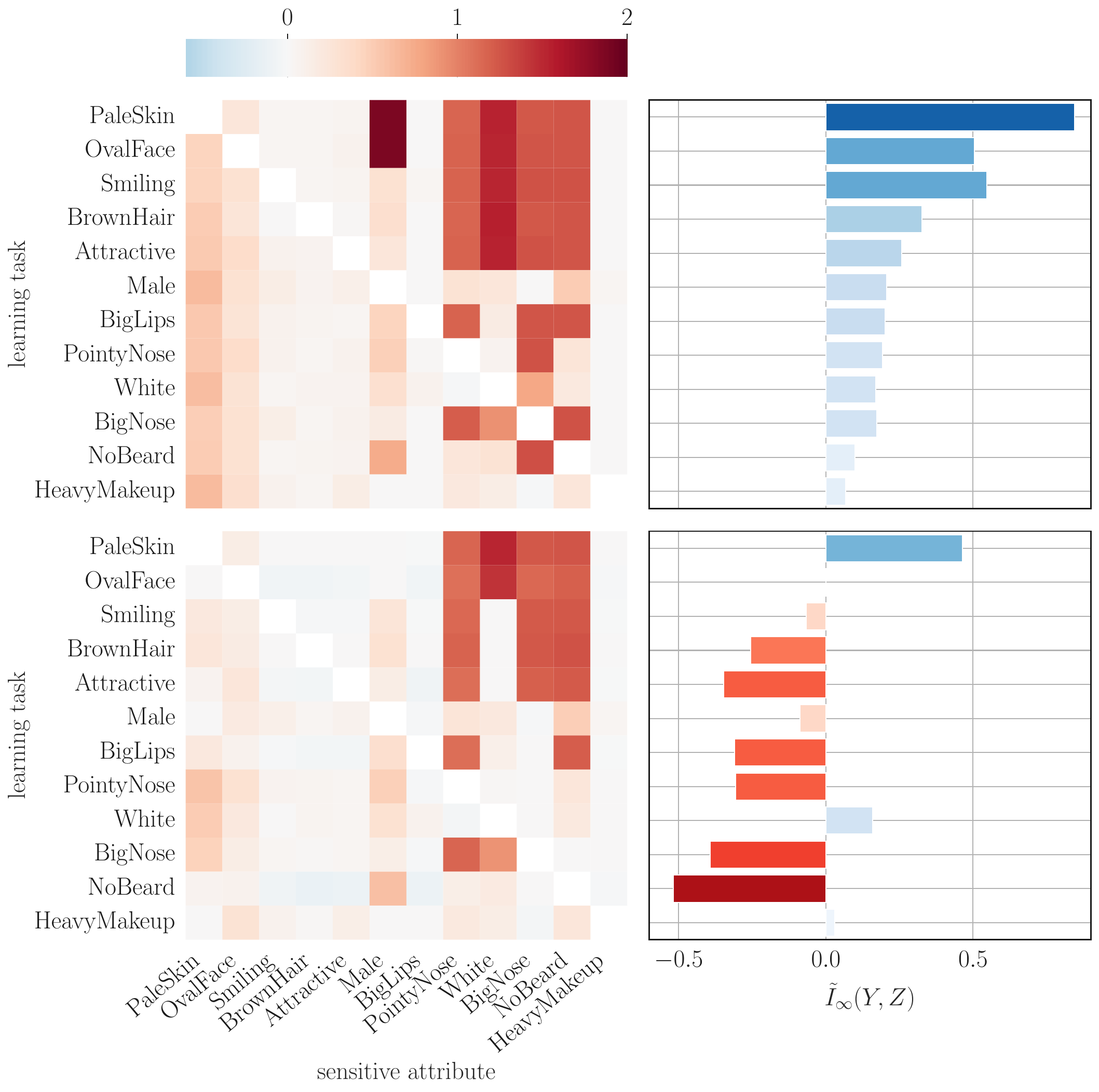}
    \caption{\textbf{If the model-generated representations have utility for the task\subfig{right}, there exists a sensitive attribute with an even higher inference gain for the adversary\subfig{left, \textcolor{red!80!black}{red} means more leakage}.} This holds for both standard ERM\subfig{top} and attribute censoring\subfig{bottom} where we censor the attribute with highest leakage in the respective ERM model (marked as $\textcolor{gray}{\tikzmarknode[strike out,draw]{2}{\text{\footnotesize \faEye}}}$). Censoring has a `whack-a-mole' effect: as we censor one attribute, leakage of another attribute increases.
    }
    \label{fig:adv_gain}
    \vspace*{-1mm}
\end{figure}

\cref{fig:adv_gain} shows the adversary's inference gain $\gain(\labelS; \modelFeatures \mid \labelY)$ which measures the adversary's classification accuracy for sensitive attribute $\labelS$ (each column) normalised by a task's (each row) fundamental leakage.

\subsection{Across Learning Techniques}
\cref{th:lpl_impossibility} implies that the strict trade-off between a representation's utility for its intended task and the LPP holds regardless of the learning technique used to obtain the feature map $\model_E(\inputFeatures) = \modelFeatures$. In \cref{subsec:emp_tradeoff}, we show that indeed even with attribute censoring through gradient reversal, an adversary can always find a data attribute for which $\Delta_{Adv} > 0$ and that thus violates the LPP.
In this section, we experimentally demonstrate that the same applies to other learning techniques that aim to hide sensitive information about the original data $\inputFeatures$, such as adversarial representation learning.




\citet{ZhaoCTG20} empirically compare the trade-off between hiding sensitive information and task accuracy of various attribute obfuscation algorithms. They find that together with gradient reversal, Maximum Entropy Adversarial Representation Learning (MAX-ENT) provides the best trade-off. We run a simple experiment on the Adult dataset~\citep{Adult}, the data used by \citeauthor{ZhaoCTG20}, that shows that the trade-off predicted by \cref{th:lpl_impossibility} also applies to the representations learned by a model trained under MAX-ENT.

\para{Experiment Setup} We use the exact same model architecture and data as \citet{ZhaoCTG20}. We train the model to predict attribute `income' and adversaries for four sensitive attributes (`age', `education', `race', and `sex'). We then calculate the utility and inference gain as described in \cref{sec:empirical}.

\begin{figure}[h!]
    \centering
	\includegraphics[width=0.5\linewidth]{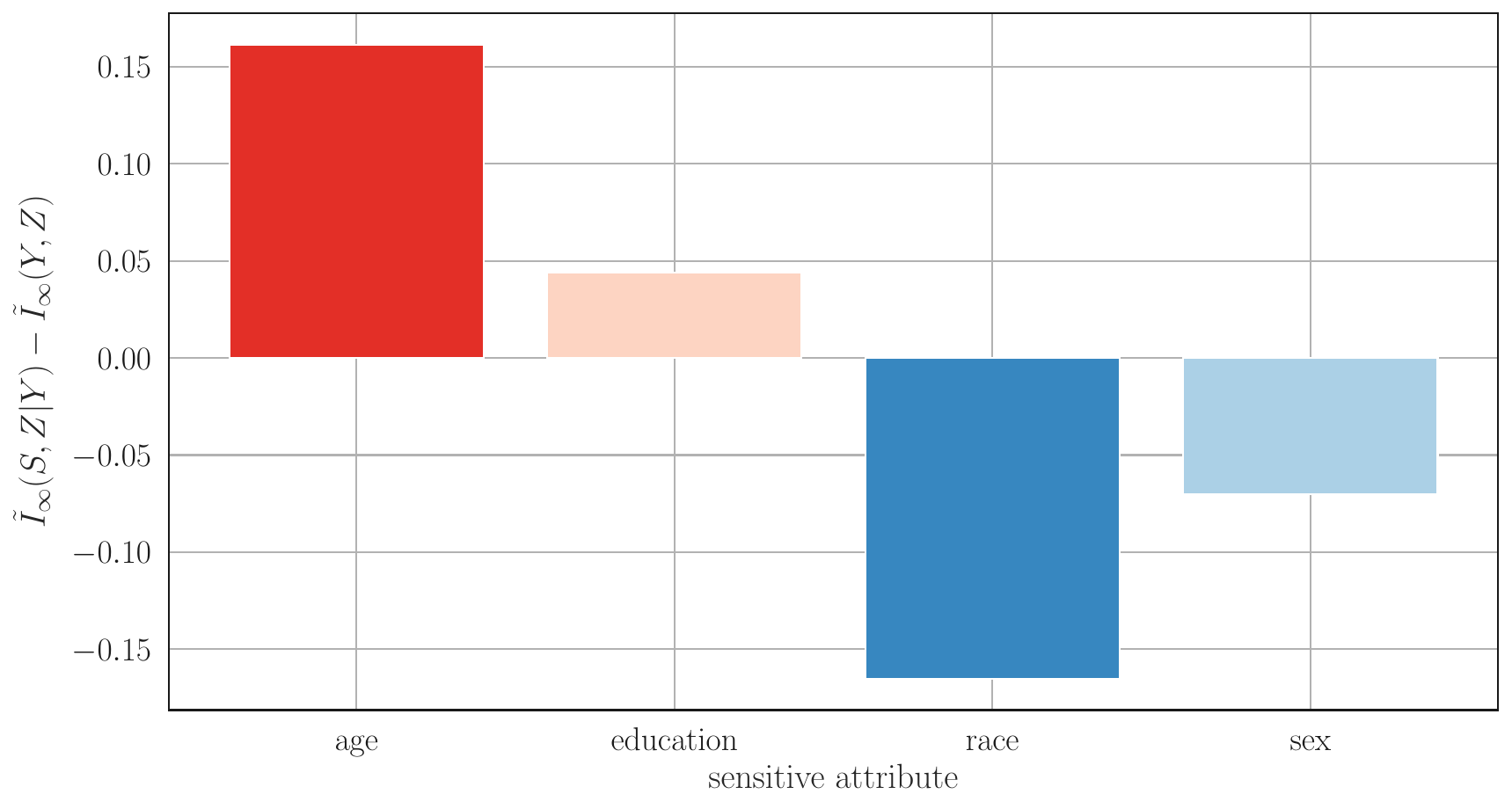}
	\caption{\textbf{The least-privilege and utility trade-off holds regardless of the representation learning technique used.} We show the delta between the adversary's inference gain and model utility for a model trained under MAX-ENT for learning task `income' on the Adult dataset. For two out of four sensitive attributes tested, the adversary's inference gain exceeds the model's utility gain.}
	\label{fig:maxent}
 \vspace*{-2mm}
\end{figure}

\cref{fig:maxent} shows that, as expected, even for a model trained under attribute obfuscation with MAX-ENT, the adversary's inference gain exceeds the model's utility gain for two out of the four sensitive attributes tested. This further supports our theoretical finding that the trade-off between LPP and utility for a prediction task of a representation applies \emph{regardless of how these representations are learned}.

\subsection{Across MLaaS settings}
We show that, as discussed in \cref{sec:background}, the fundamental trade-off of \cref{th:lpl_impossibility} equally applies to the collaborative learning setting in which users share gradients instead of raw data records to reduce unintended information leakage~\citep{MelisSDS19}. 

\para{Experiment Setup} We replicate the single batch property inference attack described by \citet{MelisSDS19}. In this setting, the feature representation $\modelFeatures = \model_E(\inputFeatures)$ shared with the service provider is a record's gradients computed across all layers of the model. As \citet{MelisSDS19} point out, these "gradient updates can [...] be used to infer feature values", i.e., are just a noisy version of the feature representations learned by the model.

\begin{figure}[h!]
    \centering
    \includegraphics{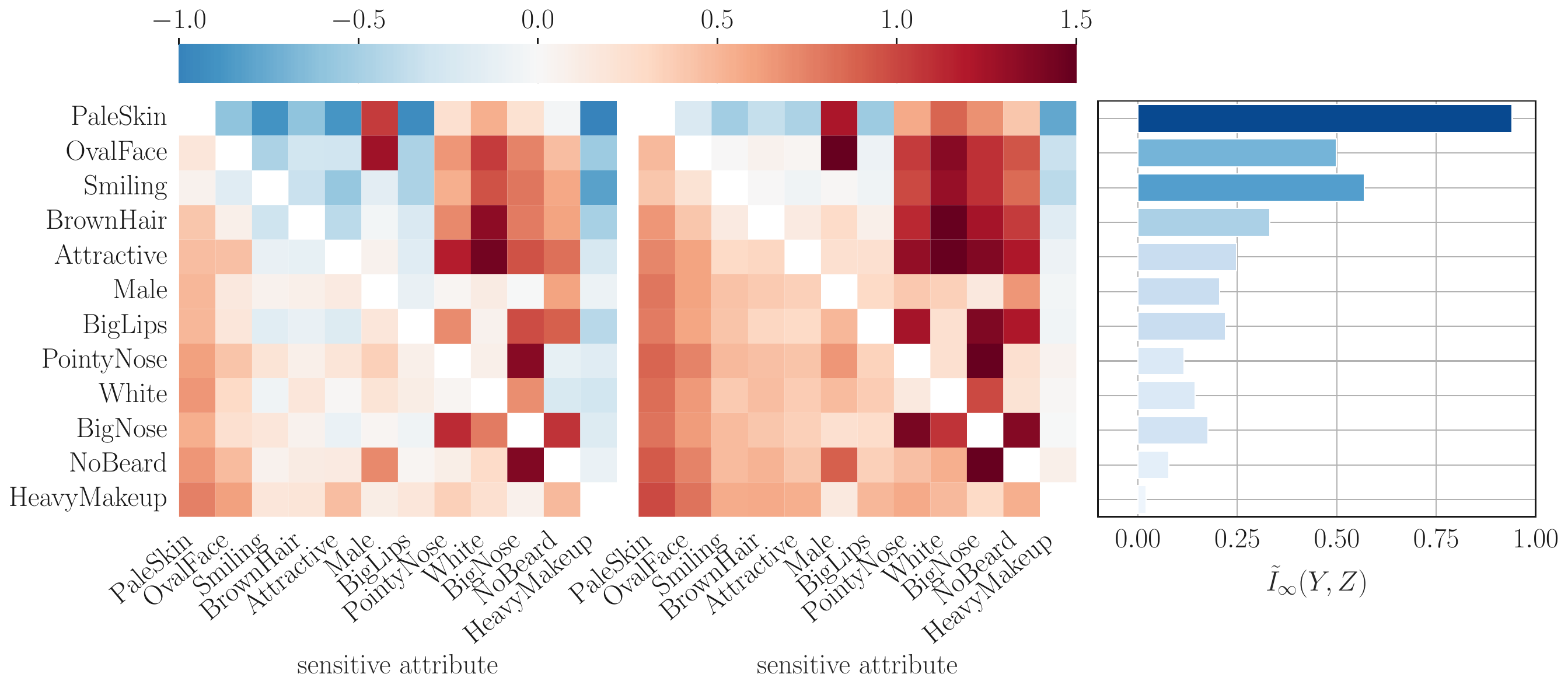}
    \caption{\textbf{The least-privilege and utility trade-off applies also the collaborative learning setting} The inference gain of an adversary that observes a record's representation in the form of gradients\subfig{left} or feature activations\subfig{middle} always violates the LPP for at least one sensitive attribute (\textcolor{red!80!black}{red} means that the adversary's inference gain is larger than the features' utility gain plotted on the \subfig{right})}
    \label{fig:gradients}
    \vspace*{-2mm}
\end{figure}

\cref{fig:gradients} confirms that an adversary that only has indirect access to a record's feature representation in the form of gradients~\subfig{left} can always find at least on sensitive attribute for which $\Delta_\mathsf{Adv} > 0$. The adversary's inference gain increases when the adversary has direct access to the feature representations \subfig{middle} but as predicted follows the fundamental trade-off predicted by \cref{th:lpl_impossibility} even when she only observes gradients.

\subsection{Across Hidden Layers}
To demonstrate that, as predicted by \cref{th:lpl_impossibility}, the strict trade-off between features' utility for a downstream prediction task and the LPP applies regardless of a model's architecture, or the structure of the feature encoder $Z = f_E(X)$, we conduct additional experiments on the LFWA+ image dataset (see \cref{sec:empirical})

To show that the trade-off holds regardless of the structure of the feature encoder, we assess the leakage of the \CNN model described in \cref{sec:empirical} across different hidden layers of the model. \cref{fig:layers} shows the delta between the adversary's inference gain and the features' utility for the intended task $\Delta_\mathsf{Adv} = \gainEst(\labelS, \modelFeatures \mid \labelY) - \gainEst(\labelY, \modelFeatures)$ across hidden layers of the model (from higher\subfig{left} to lower\subfig{right} layers). Following \cref{th:lpl_impossibility}, for every learning task and representation there always exists at least one sensitive attribute for which $\Delta_\mathsf{Adv} > 0$ and that violates the LPP. In line with prior results~\citep{MelisSDS19, MoBMH21}, we find that lower layers \subfig{right} lead to a slightly higher inference gain for the adversary; although these differences are not significant. 

\begin{figure}[h!]
    \centering
	\includegraphics[width=\figurewidth]{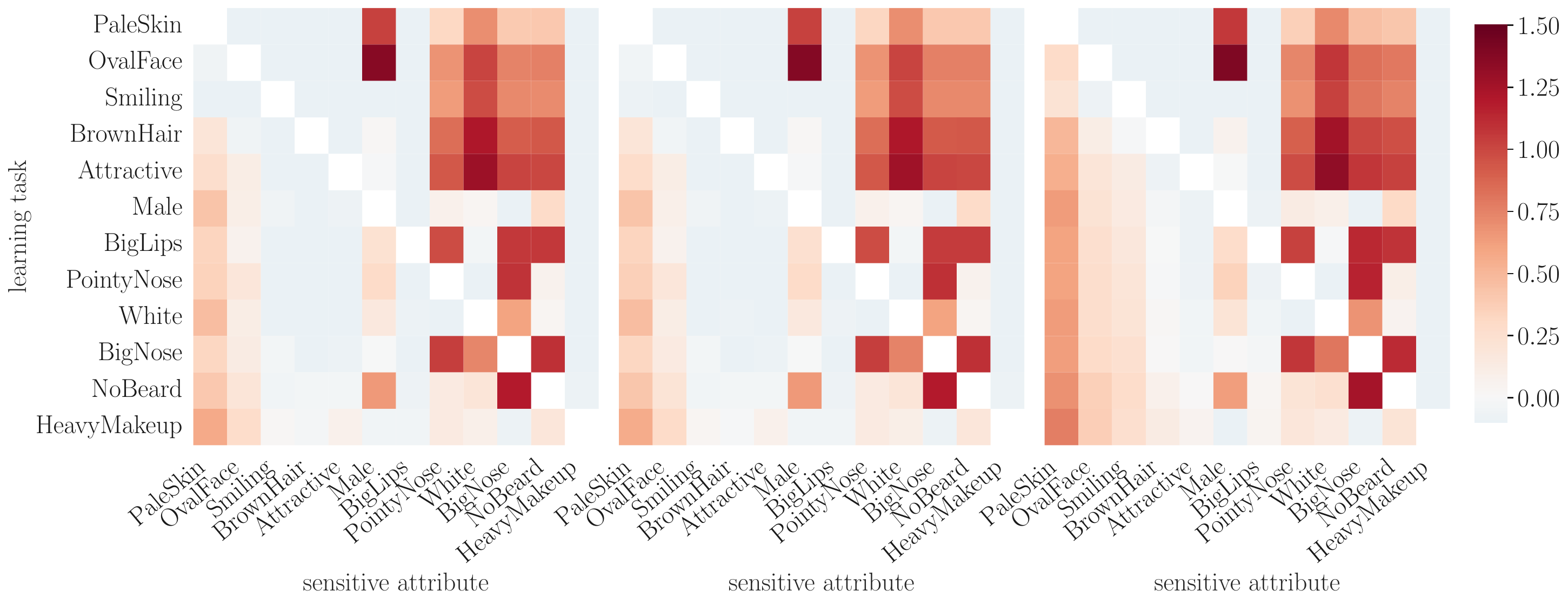}
	\caption{\textbf{The least-privilege and utility trade-off holds across hidden layers of the model} The inference gain of an adversary that observes a record's representation at the last\subfig{left}, or first\subfig{middle} and second convolutional layer\subfig{right} of a \CNN model, always violates the LPP for at least one sensitive attribute (\textcolor{red!80!black}{red} means that the adversary's inference gain is larger than the features' utility gain)}
	\label{fig:layers}
    \vspace*{-2mm}
\end{figure}

\subsection{Across Model Architecture}  To show that the trade-off holds regardless of the exact model architecture, we repeat the experiment outlined in \cref{sec:empirical} using  the ResNet-18 model architecture from \cite{HeZR15} implemented by PyTorch.
Training batch size is 32, SGD learning rate is 0.01. The adversary is given access to the feature representations at the last hidden layer of the model. 

\begin{figure}[h!]
    \centering
	\includegraphics[width=\figurewidth]{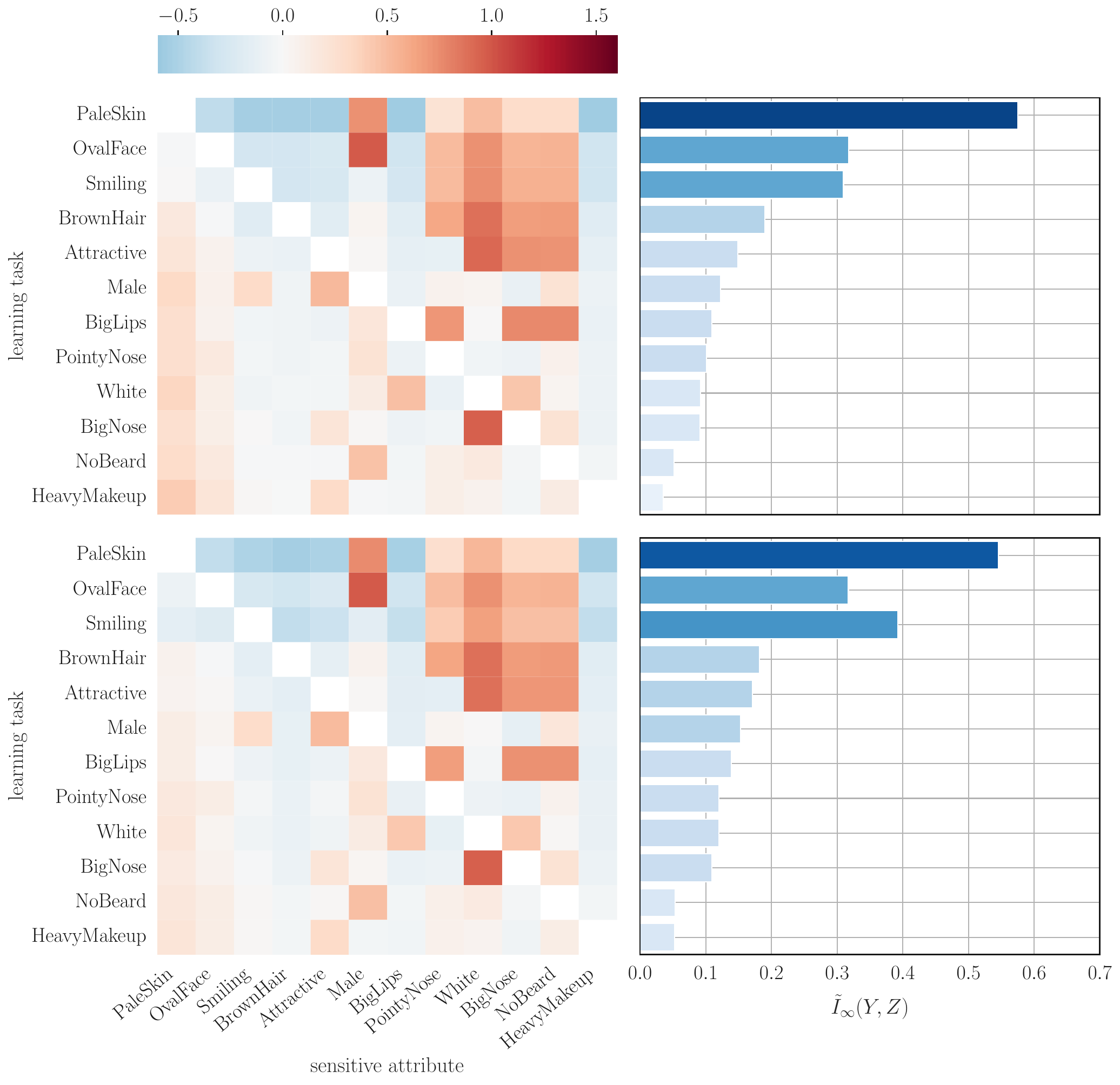}
	\caption{\textbf{The least-privilege and utility trade-off holds regardless of the model architecture.} The adversary's inference gain\subfig{left} always exceeds the utility gain\subfig{right} for at least on sensitive attribute. This hold for both a \CNN model\subfig{top} and a \ResNet model\subfig{bottom}.}
	\label{fig:resnet}
    \vspace*{-2mm}
\end{figure}

\cref{fig:resnet} compares the trade-off between utility and attribute leakage of a \CNN\subfig{top} and a \ResNet\subfig{bottom} models, both trained with standard SGD. The blue horizontal bars in \cref{fig:resnet}\subfig{right} show the model's utility for learning task $\labelY$ measured as $\gainEst(\labelY, \modelFeatures)$. The heatmaps in \cref{fig:resnet}\subfig{left} show the difference between the adversary's inference gain and the model's utility $\Delta_\mathsf{Adv} = \gainEst(\labelS, \modelFeatures \mid \labelY) - \gainEst(\labelY, \modelFeatures)$. Each row corresponds to a different learning task $\labelY$, each column represents a different sensitive attribute targeted by the adversary. We observe that regardless of the model architecture, for any learning task there always exists a sensitive attribute for which $\Delta_\mathsf{Adv}$ and thus violates the LPP.

\subsection{Across Datasets}
We ran an additional experiment to demonstrate that the strict trade-off between model utility and the LPP also holds on a very different type of dataset and model. As for tabular data, together with image data, sharing feature encodings instead of raw data is often suggested as a solution to limit harmful inferences, we choose the Texas Hospital dataset~\citep{Texas} and the TabNet model architecture~\citep{ArikP21} for these experiments.

\para{Data} The Texas Hospital Discharge dataset~\citep{Texas} is a large public use data file provided by the Texas Department of State Health Services. The dataset we use consists of 5,202,376 records uniformly sampled from a pre-processed data file that contains patient records from the year 2013. We retain $18$ data attributes of which $11$ are categorical and $7$ continuous.

\para{Experiment Setup}  In each experiment, we select one attribute as the model's learning task $\labelY$ and a second attribute as the sensitive attribute $\labelS$ targeted by the adversary.
We repeat each experiment $5$ times to capture randomness of our measurements for both the model and adversary, and show average results across all $5$ repetitions.
At the start of the experiment, we split the data into the three sets $\dataset_T$, $\dataset_E$, and $\dataset_A$. We train a TabNet model on the train set $\dataset_T$ for the chosen learning task and then estimate the model's utility on the evaluation set $\dataset_E$.  We measure the utility of the model-generated representations as the multiplicative gain $\gainEst(\labelY; \modelFeatures)$.
After model training and evaluation, we train both the label-only and features adversary on the auxiliary data $\dataset_A$. The features adversary is given access to a record's representation at the last encoding layer of the TabNet encoder (see \cite{ArikP21} for details of the model architecture). For a given sensitive attribute $\labelS$, we estimate the adversary's gain as $\gainEst(\labelS, \modelFeatures \mid \labelY)$.

\begin{figure}[h!]
    \centering
	\includegraphics[width=\figurewidth]{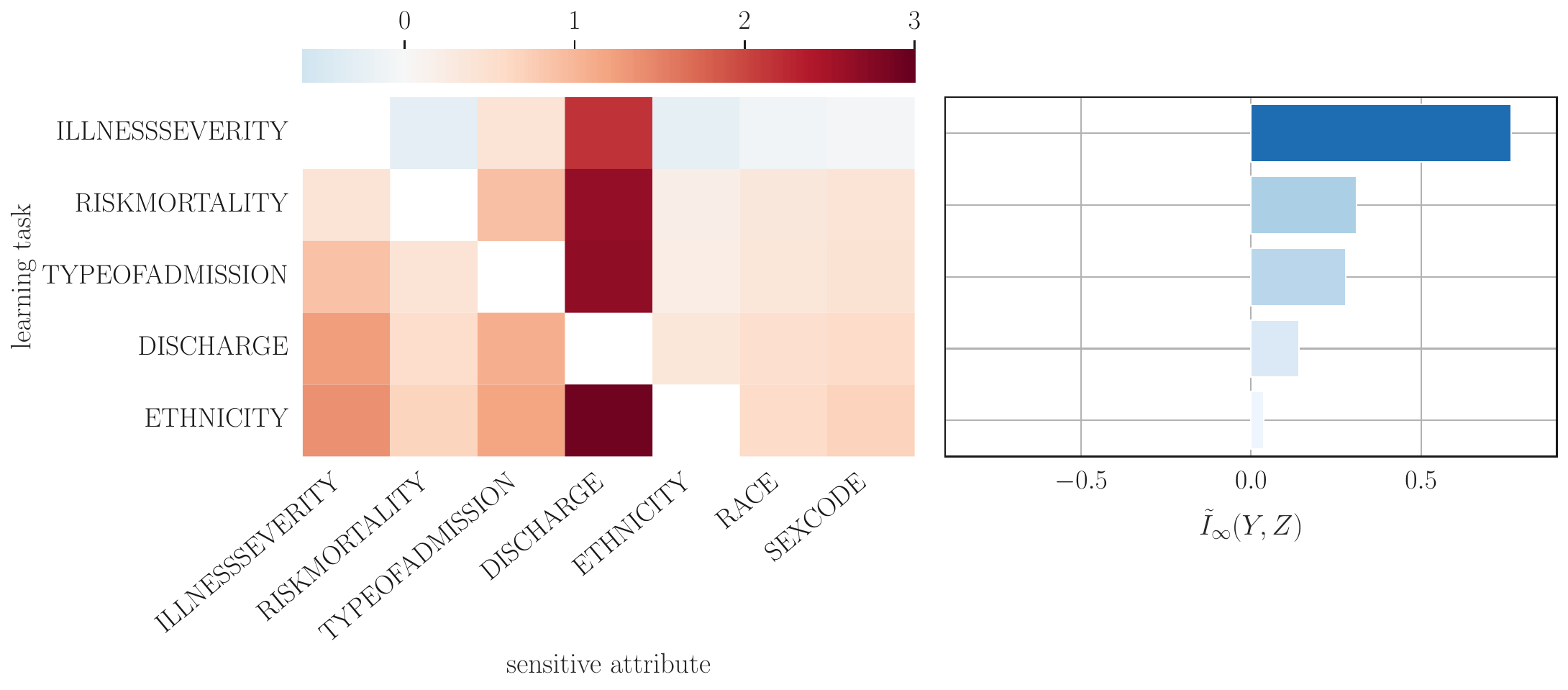}
	\caption{Attribute leakage\subfig{left} and model utility\subfig{right} for a TabNet model trained on the Texas Hospital dataset}
	\label{fig:tabnet}
\end{figure}

As above, the bar chart in \cref{fig:tabnet}\subfig{right} shows the model's utility for learning task $\labelY$ indicated in each row measured as $\gainEst(\labelY, \modelFeatures)$. The heatmaps in \cref{fig:tabnet}\subfig{left} show the difference between the adversary's inference gain and the model's utility $\Delta_{Adv} = \gainEst(\labelS, \modelFeatures \mid \labelY) - \gainEst(\labelY, \modelFeatures)$. As on the LFWA+ dataset, for any learning task there always exists a sensitive attribute for which an adversary gains an advantage from observing a target record's feature representation and $\Delta_\mathsf{Adv} > 0$. This demonstrates the strict trade-off between utility and the LPP as predicted by \cref{th:lpl_impossibility}.

%
%
%

\end{document}



